\documentclass[conference]{IEEEtran}
\usepackage{times}
\usepackage[activate={true,nocompatibility},final,tracking=true,kerning=true,spacing=true,factor=1100,stretch=10,shrink=10]{microtype}
\usepackage[bookmarks=true]{hyperref}
\usepackage[ruled, vlined, linesnumberedhidden]{algorithm2e}
\usepackage{eucal}
\usepackage{dsfont}
\usepackage{amsmath}
\usepackage{amsfonts}

\usepackage{amsthm}
\usepackage{mathtools}
\usepackage{physics}
\usepackage{lipsum}
\usepackage{nicefrac}
\usepackage{booktabs}
\usepackage{multirow}
\usepackage{flushend}
\usepackage{makecell}
\usepackage{tabularx}

\usepackage[table]{xcolor}  % for \rowcolor

\let\labelindent\relax
\usepackage{enumitem}
\usepackage{thmtools}

\usepackage[numbers, sort&compress]{natbib}
\usepackage{multicol}
\usepackage{cleveref}

\crefformat{equation}{(#2#1#3)}
\crefrangeformat{equation}{(#3#1#4) to~(#5#2#6)}
\crefmultiformat{equation}{(#2#1#3)}%
{ and~(#2#1#3)}{, (#2#1#3)}{ and~(#2#1#3)}

\usepackage{caption}
\usepackage{eucal}
\usepackage{dsfont}
\usepackage{xspace}
\usepackage{xifthen}

\newcommand{\set}[1]{\mathbf{#1}}
\newcommand{\func}[1]{\mathds{#1}}
\newcommand{\oper}[1]{\mathcal{#1}}
\newcommand{\dist}[1]{\mathcal{#1}}

\DeclareMathOperator*{\argmin}{argmin}

\DeclareMathOperator*{\minimize}{minimize}

\newcommand{\Sink}{\func{T}_\varepsilon}
\newcommand{\EOT}{\func{T}_\varepsilon}

\newcommand{\KL}{\func{D}}

\newcommand{\inner}[2]{\left\langle #1, #2 \right\rangle}
\newcommand{\transp}[1]{{#1}^{\mathsf{T}}}

\newcommand{\R}{\set{R}}
\newcommand{\E}{\func{E}}
\newcommand{\algo}[1]{\texttt{#1}}

\newcommand{\scd}[1][]{\ifthenelse{\isempty{#1}}%
        {SCD\xspace}% if #1 is empty
        {SCD\xspace}% if #1 is not empty}
        }
\newcommand{\mppi}{MPPI\xspace}

\newcommand{\opt}[1]{{#1}^{\star}}
\newcommand{\inv}[1]{{#1}^{-1}}
\newcommand{\define}{\coloneqq}

\usepackage{pifont}         % for checkmarks

\newcommand{\seq}[1]{\vb{#1}}

\newtheorem{remark}{Remark}
\newtheorem{proposition}{Proposition}

\IEEEoverridecommandlockouts 

\begin{document}

% paper title
%\title{Sampling-Based Control via Entropy-Regularized Optimal Transport}

\title{Sampling-Based Control via Entropy-Regularized Optimal Transport}
% You will get a Paper-ID when submitting a pdf file to the conference system
%\author{Author Names Omitted for Anonymous Review. Paper-ID 1011}
\author{Vincent Pacelli,\IEEEauthorrefmark{2} Akash Ratheesh,\IEEEauthorrefmark{2}  Evangelos A. Theodorou\\
Autonomous Control and Decision Systems Laboratory\\
Georgia Institute of Technology\thanks{\IEEEauthorrefmark{2}Equal Contribution}}

\maketitle

\begin{abstract}
Sampling-based model predictive control methods like MPPI and CEM are essential for real-time control of nonlinear robotic systems, particularly where discontinuous dynamics preclude gradient-based optimization. However, these methods derive from information-theoretic objectives that are agnostic to the geometry of the control problem, leading to pathological behaviors such as mode-averaging when the cost landscape is complex. We present OT-MPC\footnote{Project Page: https://acdslab.github.io/ot-mpc/}, a sampling-based algorithm that overcomes these limitations through an entropy-regularized optimal transport formulation. By computing an optimal coupling between candidate control sequences and low-cost proposals, OT-MPC refines candidates toward nearby promising samples while coordinating updates across the ensemble to maintain coverage of the solution space. We derive closed-form, gradient-free updates via the Sinkhorn algorithm, enabling real-time performance. Experiments on navigation, manipulation, and locomotion tasks demonstrate improved success rates over existing methods.

\end{abstract}

\IEEEpeerreviewmaketitle

\section{Introduction}
\label{sec:intro}
Sampling-based model predictive control is a workhorse for real-time control of nonlinear and contact-rich robotic systems. Algorithms like Model Predictive Path Integral (\algo{MPPI})~\cite{Williams17, Williams18} and the Cross-Entropy Method (\algo{CEM})~\cite{Rubinstein99, Kobilarov11, Pinneri20} leverage parallel simulation to optimize complex cost functions by sampling and scoring candidate trajectories. Unlike gradient-based methods, they require only the ability to evaluate trajectory costs---making them compatible with black-box simulators and learning-based models. This flexibility makes them a common choice for manipulation and locomotion~\cite{Abraham20, Xue24, Hansen24, Alvarez25, Keshavarz25}, where gradients are unavailable or expensive.

Despite their success, the information-theoretic foundation of these methods leads to fundamental limitations. \algo{MPPI} aggregates cost information by taking a weighted average over all samples, with weights given by exponentiated trajectory costs. Since this average ignores where samples lie in the space, it produces \emph{mode-averaging}: the resulting control does not represent any local minimum, but a blend of multiple minima. For instance, a robot navigating around an obstacle will average trajectories on either side---steering directly into a collision. \algo{CEM} avoids mode-averaging through elite selection, but this mechanism induces \emph{mode-seeking} behavior that commits aggressively to one mode and limits exploration.

\begin{figure}[t!]
\centering
\includegraphics[width=\columnwidth]{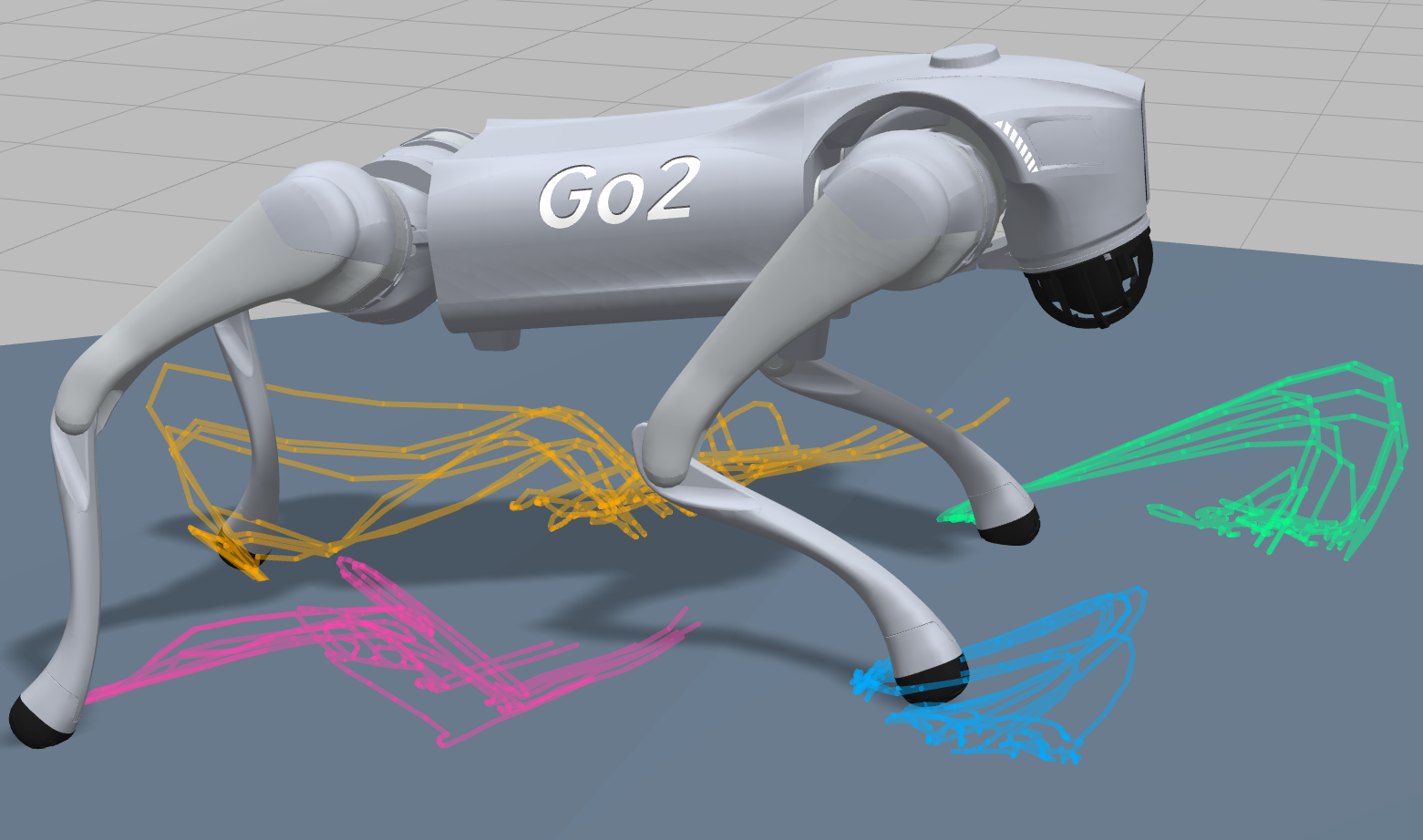}%
  \caption{The proposed \algo{OT-MPC} algorithm controlling a Unitree Go2 quadruped. Colored curves show planned foot trajectories from different candidate solutions at a key decision point  where multiple gait strategies are viable. \algo{OT-MPC} naturally maintains diverse candidates when the cost landscape is multimodal; when candidates converge to a single mode, local proposal sampling concentrates refinement where it is needed.}%
  \label{fig:teaser}%
\end{figure}

Both failure modes stem from the control-as-inference formulation underlying these methods, which frames optimal control as sampling from a Gibbs distribution over trajectories. The resulting objective is an information-theoretic divergence---typically the KL divergence---which quantifies \emph{how} probability mass is distributed but not \emph{where}. Methods derived from such objectives therefore aggregate sample information \emph{globally}, without regard to spatial arrangement. The result is no inherent mechanism for local refinement, mode preservation, or ensemble coordination. Prior work has sought to address these shortcomings, but existing approaches are typically \emph{post hoc} heuristic modifications~\cite{Abraham20, Bhardwaj22, Xue24, Yi24, Okada20}.

This article addresses these shortcomings at a foundational level by developing \textbf{Sinkhorn Coordinate Descent}~(\algo{SCD}), a sampling-based optimization algorithm derived from an optimal transport~(OT) variational principle. Unlike information-theoretic divergences, OT objectives such as the Wasserstein distance measure not only whether two distributions assign probability mass similarly, but also the cost of transforming one into the other---incorporating spatial information. Entropy regularization softens the coupling and enables efficient computation via the Sinkhorn algorithm~\cite{Cuturi13, Peyre19}. The resulting MPC algorithm, \textbf{Optimal Transport MPC}~(\algo{OT-MPC}), computes an optimal coupling between particles and low-cost proposals, then updates each particle toward its weighted barycenter. Like \algo{MPPI}, this algorithm requires only cost evaluations, making it compatible with non-smooth dynamics and non-differentiable costs. Experiments on navigation, manipulation, and locomotion demonstrate improved success rates over existing methods.

\textbf{\textit{Contributions.}} In summary, this article contributes the following  to the theory and practice of sampling-based control:
\begin{enumerate}[leftmargin=*, itemsep=2pt, topsep=2pt, label=\arabic*.]
    \item We propose \algo{SCD}, a gradient-free optimization algorithm, and its MPC instantiation, \algo{OT-MPC}. Unlike other sampling-based methods, \algo{SCD} updates particles based on both cost and geometric proximity---enabling local refinement while preserving diversity.
    \item We establish theoretical properties of \algo{SCD}, including monotone descent, convergence guarantees, and closed-form updates for quadratic costs.
    \item We demonstrate empirically that \algo{OT-MPC} achieves higher success rates than \algo{MPPI}, \algo{CEM}, and \algo{SV-MPC} on challenging navigation, manipulation, and locomotion tasks.
\end{enumerate}
%\lipsum[1-8]

\section{Related Work}
\label{sec:related work}
% Related Work Section for Sinkhorn MPC
% Include with \input{related_work}

% \begin{table}[t]
% \centering
% \small
% \setlength{\tabcolsep}{4pt}
% \begin{tabular}{@{}lcccc@{}}
% \toprule
% \textbf{Method} & \textbf{Gradient-free} & \textbf{Update rule} & \textbf{Multimodal} \\
% \midrule
% \mppi{} & \cmark & Global weighted avg. & \xmark\ \\[2pt]
% \cem{} & \cmark & Elite Selection & \xmark\ \\[2pt]
% \cmaes{} & \cmark & Covariance adaptation & \cmark \\[2pt]
% \svgd{} & \xmark & Kernel gradient & \cmark \\[2pt]
% \midrule
% \textbf{\scd[0]} (Ours) & \cmark & \textbf{Per-particle OT} & \cmark \\
% \bottomrule
% \end{tabular}
% \caption{\mdseries\upshape Comparison of sampling-based optimization methods. \scd[0] is gradient-free while preserving multimodal structure via optimal transport.}
% \label{tab:method comparison}
% \end{table}

The \algo{OT-MPC} algorithm inherits the control-as-inference formulation common to sampling-based MPC but replaces the information-theoretic objective with an optimal transport one. We first review variational inference methods and then discuss prior uses of optimal transport in robotics.

\subsection{Variational Inference for Model-Predictive Control}

Sampling-based MPC methods such as \algo{MPPI}~\cite{Williams17, Williams18} and \algo{CEM}~\cite{Rubinstein99, Kobilarov11, Pinneri20} frame optimal control as variational inference, approximating a Gibbs distribution over low-cost trajectories~\cite{Attias03, Lambert21, Wang21}. Both methods use information-theoretic objectives---\algo{MPPI} minimizes KL divergence via importance sampling, while \algo{CEM} fits a parametric distribution to elite samples. As discussed in the introduction, the former leads to mode-averaging and the latter to mode-seeking. Recent variants address these issues through annealing schedules~\cite{Xue24}, covariance adaptation~\cite{Yi24}, warm-starting~\cite{Bhardwaj22}, and mixture models~\cite{Okada20}, but these are heuristics that do not change the variational objective.

The closest work conceptually is the Tsallis VI-MPC~\cite{Wang21}, which also modifies the divergence in the inference formulation. That work unifies \algo{MPPI} and \algo{CEM} through a generalized entropy objective that interpolates between mode-averaging and mode-seeking behavior. Similarly, \algo{OT-MPC} replaces the KL divergence, but with an optimal transport objective that incorporates spatial information rather than adjusting the entropy's tail behavior. Other VI-MPC methods leverage gradient information: Stein variational approaches~\cite{Lambert21, Honda24} use the score function to update particles, and recent DDP-based methods~\cite{Aoyama24} use second-order derivatives. Gradient-free variants of \algo{SVGD} exist~\cite{Han18} but require careful selection of an auxiliary distribution and suffer from high-variance updates. In contrast, \algo{OT-MPC} is zeroth-order and only requires  cost evaluations.

\subsection{Optimal Transport in Robotics and Control}
Optimal transport has been used in robotics to formulate control objectives---steering multi-agent systems to goal configurations~\cite{Ito23}, covariance steering for robust planning~\cite{Okamoto19, Yin22, Saravanos24, Ratheesh25}, and imitation learning~\cite{Dadashi21, Papagiannis22}. These applications use OT to define \emph{what} to achieve; \algo{OT-MPC} instead uses OT to determine \emph{how} to solve the control problem.

Diffusion models~\cite{Ho20}, now widely used for policy learning~\cite{Janner22, Chi25}, can be viewed as solving an optimal transport problem between noise and data distributions. However, diffusion transports particles independently via a learned score function. \algo{SCD} instead computes an explicit coupling between particles and proposals, with marginal constraints that coordinate updates and prevent mode collapse. 

The most similar existing algorithm to \algo{OT-MPC} is \algo{MPOT}~\cite{Le23}, which also uses the Sinkhorn algorithm for trajectory optimization. The key distinction is where the task objective enters the formulation: \algo{MPOT} encodes the control cost directly in the transport cost matrix, steering waypoints toward globally low-cost regions, and \algo{OT-MPC} encodes costs through the marginal weights and reserves the transport cost for geometric proximity. This separation enables particles to move toward nearby promising proposals rather than distant optima, providing the local refinement that avoids mode-averaging and respects the geometry of the space.

\section{Background: Optimal Transport}
\label{sec:background}
% Background Section for Sinkhorn MPC
% Include with \input{sections/background}

%\begin{figure*}
%    \centering
%    \includegraphics[width=0.9\linewidth]{images/multi_modal_quad.png}
%    \caption{Explain Algorithm}
%    \label{fig:placeholder}
%\end{figure*}

\begin{figure*}
    \centering
    \includegraphics[width=1\linewidth]{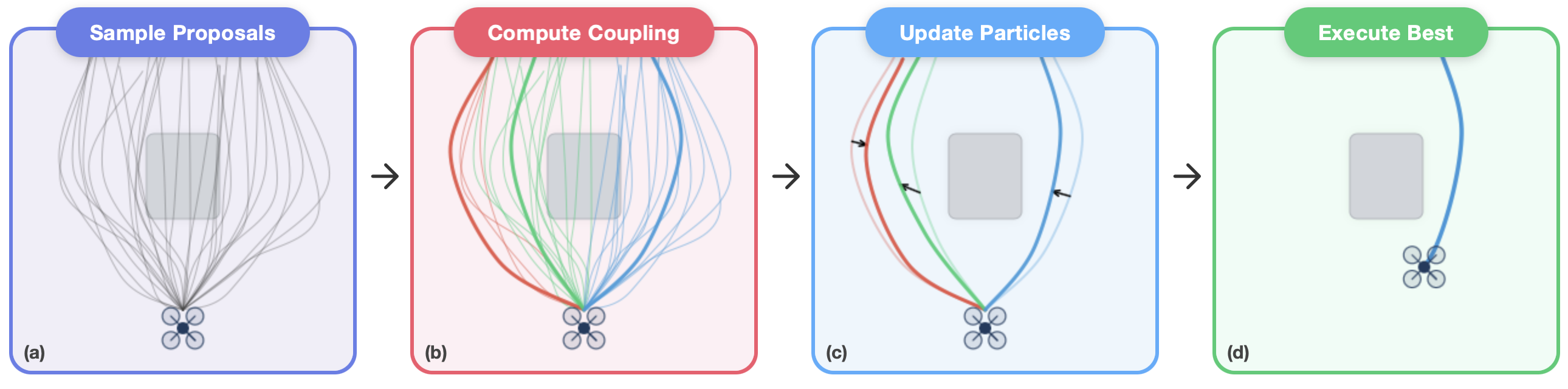}
    \caption{\textbf{Overview of \algo{OT-MPC}}. \textbf{(a)} Proposals are sampled to explore the trajectory space. \textbf{(b)} The Sinkhorn algorithm computes a \emph{soft coupling} between candidates (bold curves) and proposals factoring in both cost and proximity. Colors indicate which to which proposals the candidate is most strongly coupled. \textbf{(c)} Each candidate updates toward its coupled proposals via a \emph{barycentric projection} \cref{eq:barycentric}---refining locally while preserving distinct modes. \textbf{(d)} The lowest-cost candidate is executed. Unlike \algo{MPPI}, which averages all samples globally, \algo{OT-MPC} couples each candidate to nearby proposals---avoiding the mode-averaging that would steer directly into the obstacle and enabling local refinement of candidates within each mode.}
    \label{fig:overview}
\end{figure*}

Optimal transport (OT) finds the minimum-cost coupling between two distributions~\cite{Villani08, Peyre19}. For probability mass functions $q \in \set{\Delta}^N$ and $p \in \set{\Delta}^M$, a coupling is a joint distribution $\Gamma \in \R_+^{N \times M}$ with marginals $\Gamma 1_M = q$ and $\transp{\Gamma} 1_N = p$; denote this set $\set{\Gamma}(q, p)$. Given a cost matrix $C \in \R^{N \times M}$, the OT problem minimizes total transport cost. When the cost reflects distance, OT metrizes the space of distributions---for instance, yielding the Wasserstein distance when $C_{ij} = \|x_i - y_j\|^2$. Unlike the KL divergence, which compares distributions pointwise without regard to the underlying space, OT incorporates geometric structure through the cost matrix.

The OT problem is a linear program with $\mathcal{O}(N^3 \log N)$ complexity, but entropy regularization enables efficient approximate solutions via the Sinkhorn algorithm~\cite{Cuturi13}. The entropic OT (EOT) problem is:
\begin{align}
    \Sink(q, p) = \minimize_{\Gamma \in \set{\Gamma}(q,p)} \sum_{i,j} C_{ij} \Gamma_{ij} - \varepsilon H(\Gamma), \label{eq:eot}
\end{align}
where $H(\Gamma) \define -\sum_{i,j} \Gamma_{ij}(\log \Gamma_{ij} - 1)$ is the entropy. The unique solution has the form $\opt{\Gamma}_{ij} = u_i K_{ij} v_j$ with $K_{ij} = \exp(-C_{ij}/\varepsilon)$, where the scaling vectors $u, v$ are computed by alternating projections onto the marginal constraints (\Cref{alg:sinkhorn}). As $\varepsilon \to 0$, the coupling becomes sparse, deterministic coupling $\opt{\Gamma}$ as converges to the unregularized OT solution. As $\varepsilon \to \infty$, the coupling ignores the transport cost and $\opt{\Gamma} \to q \transp{p}$. 

\begin{algorithm}[t]
\caption{Sinkhorn Algorithm (\algo{Sink})}
\label{alg:sinkhorn}
\KwIn{Cost matrix $C \in \R^{N \times M}$, marginals $q \in \R^N$, $p \in \R^M$, regularization $\varepsilon > 0$.}
\KwOut{Optimal coupling $\opt{\Gamma} \in \R^{N \times M}$.}
\BlankLine
$K_{ij} \gets \exp(-C_{ij}/\varepsilon)$\;
$u, v \gets 1_N, 1_M$\ \tcp*[r]{Can warm start}
\Repeat{$\|\Gamma 1_M - q\|_1 < \delta$}{%
$u \gets q\,/\,(K v)$\;
$v \gets p\,/\,(\transp{K} u)$\;
}
$\opt{\Gamma}_{ij} \gets u_i K_{ij} v_j$\;
\end{algorithm}

\section{Optimal Control Problem Formulation}
\label{sec:problem formulation}
This section formulates the optimal control problems solved by \algo{OT-MPC} using the control-as-inference perspective~\cite{Attias03, Lambert21, Wang21}, which is mathematically equivalent to the free energy duality in the \algo{MPPI} literature~\cite{Williams17, Williams18}.

Consider a deterministic discrete-time system with state $x_t \in \R^n$, control $u_t \in \R^m$, initial condition $x_0$, and dynamics:
\begin{align}
    x_{t+1} = F(x_t, u_t), \quad t = 0, \dots, t_f. \label{eq:dynamics}
\end{align}
Let $\set{X}$ and $\set{U}$ denote the sets of state and control sequences of horizon $t_f$. The task is specified by a cost function \mbox{$J: \set{X} \times \set{U} \to \R_+$}. A common choice is,
\begin{align}
J(\seq{x}, \seq{u}) = \sum_{t=0}^{t_f} \ell(x_t) + \frac{1}{2}\transp{u}_t R u_t,
\end{align}
though $J(\seq{x}, \seq{u})$ need not be continuous or differentiable. Let $\Phi(\seq{u}; x_0)$ denote the \emph{rollout} map, which returns the state-control trajectory satisfying~\cref{eq:dynamics}. The optimal control problem is:
\begin{align}
    \minimize_{\seq{u} \in \set{U}} \quad S(\seq{u}; x_0) \define (J \circ \Phi)(\seq{u}; x_0). \tag{P} \label{eq:ocp}
\end{align}

\subsection{Optimal Control via Variational Inference}
The \algo{OT-MPC} algorithm solves \cref{eq:ocp} using \algo{SCD} (\Cref{sec:scd}), which approximately samples from a target distribution that concentrates probability mass at the minima of $S(\seq{u}; x_0)$. This target is derived using the control-as-inference framework, which reformulates \cref{eq:ocp} as a Bayesian inference problem in which $\seq{u}$ are the latent variables to be inferred.

Select a prior $P(\seq{u})$ and define a binary random variable $o \in \{0, 1\}$ to indicate whether controls $\seq{u}$ are optimal. Optimal sequences can be generated by sampling from the posterior:
\begin{align}
    \opt{P}(\seq{u}) \coloneqq P(\seq{u} | o = 1) = \frac{P(o = 1|\seq{u})P(\seq{u})}{P(o = 1)}.
\end{align}
Since sampling from this posterior is intractable, we seek a variational approximation by finding the distribution in a tractable family $\set{Q} \subseteq \set{\Delta}(\set{U})$ closest in KL divergence:
\begin{align}
    \minimize_{Q \in \set{Q}}\quad \KL\qty[Q; \opt{P}] = \E_Q\qty[\log\frac{P(o = 1)Q(\seq{u})}{P(o = 1|\seq{u})P(\seq{u})}].
\end{align}
Rearranging yields the \emph{evidence lower bound} (ELBO):
\begin{align}
    \sup_{Q \in \set{Q}}\ \E_Q[\log P(o = 1 | \seq{u})] - \KL\qty[Q; P] \leq \KL\qty[Q; \opt{P}].\label{eq:elbo}
\end{align}
The standard choice of likelihood is the exponentiated cost:
\begin{align}
    P(o = 1 | \seq{u}) \propto \exp(-\beta S(\seq{u})),
\end{align}
under which the ELBO becomes:
\begin{align}
    \minimize_{Q \in \set{Q}}\ \E_Q[S(\seq{u})] + \frac{1}{\beta} \KL\qty[Q; P]. \label{eq:nonequilibrium}
\end{align}
When $\set{Q} = \set{\Delta}(\set{U})$, the solution is the Gibbs measure:
\begin{align}
    \opt{Q}(\seq{u}) \propto \exp(-\beta S(\seq{u}))P(\seq{u}).\label{eq:gibbs}
\end{align}
The inverse temperature $\beta$ controls concentration: as $\beta \to 0$, $Q \to P$; as $\beta \to \infty$, probability concentrates at the global minima of $S$. In practice, $\opt{Q} \notin \set{Q}$, so algorithms like \algo{OT-MPC} and \algo{MPPI} approximate it within $\set{Q}$.

\subsection{The Path Integral Method for Sampling Controls}
\label{ssec:mppi}

This section briefly describes how \algo{MPPI} solves \cref{eq:nonequilibrium}. The objective~\cref{eq:nonequilibrium} is equivalent to the free energy variational inequality in the \algo{MPPI} literature~\cite{Williams17, Williams18}. The distinction between the two inequalities is only in terminology.

The \algo{MPPI} algorithm restricts $\set{Q}$ to Gaussian distributions with fixed covariance:
\begin{align}
    u_t = \bar{u}_t + \delta u_t, \quad \delta u_t \sim \dist{N}(0, \Sigma_t),
\end{align}
where $\seq{\Sigma} = (\Sigma_t)_{t=0}^{t_f}$ is a known covariance sequence and $\bar{\seq{u}} = (\bar{u}_t)_{t=0}^{t_f}$ is the mean to be optimized. Since $\opt{Q}$ cannot be sampled directly, importance sampling approximates the minimum mean square error (MMSE) estimator:
\begin{align}
    \E_{\opt{Q}}[\seq{u}] &= \frac{\E_Q\qty[\seq{u}\exp(-\beta S(\seq{u}))]}{\E_Q\qty[\exp(-\beta S(\seq{u}))]}\label{eq:mppi is}\\
    &\approx \sum_{j=1}^M \underbrace{\frac{\exp(-\beta S(\seq{u}_j))}{\sum_{k=1}^M \exp(-\beta S(\seq{u}_k))}}_{w_j \define} \cdot \seq{u}_j, \nonumber
\end{align}
where $\seq{u}_j = \bar{\seq{u}} + \delta\seq{u}_j$ with $\delta\seq{u}_j \sim \dist{N}(0, \seq{\Sigma})$. This yields the \algo{MPPI} update:
\begin{align}
    \bar{\seq{u}}^{(l+1)} \gets \bar{\seq{u}}^{(l)} + \sum_{j=1}^M w_j \cdot \delta\seq{u}_j.\label{eq:mppi}
\end{align}

\section{Sinkhorn Coordinate Descent}
\label{sec:scd}
% Method Section: Sinkhorn Coordinate Descent
% Include with \input{sections/method}

This section describes \emph{Sinkhorn Coordinate Descent} (\algo{SCD}), a gradient-free algorithm that evolves $N$ particles \mbox{$\seq{z} = (z_i)_{i=1}^N$} toward the target distribution $\opt{Q}(z)$ using $M$ proposals \mbox{$\seq{y} = (y_j)_{j=1}^M$}  sampled from a reference distribution $R(\seq{y}|\seq{x})$. Note that the reference may optionally depend on the particle values. Unlike importance sampling, which computes a single global average, \algo{SCD} incorporates geometric information through the EOT cost computed via \Cref{alg:sinkhorn}. This section presents \algo{SCD} independently of the control setting; \cref{sec:ot-mpc} instantiates it for MPC.

The proposals are sampled from a distribution $R(\seq{y}|\seq{z})$, which is optionally conditioned on the particle values. The target marginal $p(\seq{y}) \in \set{\Delta}^M$ is defined via self-normalizing importance sampling:
\begin{align}
    p_j(\seq{y}) = \frac{\opt{Q}(y_j)}{\sum_{k=1}^M \opt{Q}(y_k)}.
\end{align}
The particle marginal $q(\seq{z}) \in \set{\Delta}^N$ can be set arbitrarily, e.g., it can be defined analogously to $p(\seq{y})$ or a uniform distribution to encourage exploration. Together with a cost function $c: \R^n \times \R^n \to \R$, these define an EOT problem over particle positions:
\begin{align}
    \oper{L}^c_{\varepsilon}(\seq{z}, \Gamma; \seq{y}) &\define \sum_{i,j} c(z_i, y_j) \Gamma_{ij} - \varepsilon H(\Gamma)\\
    \EOT^c(\seq{z}, \seq{y})\ &\define\ \minimize_{\Gamma \in \set{\Gamma}(q, p)}\quad \oper{L}^c_{\varepsilon}(\seq{z}, \Gamma; \seq{y}).\label{eq:eot part}
\end{align}

The \algo{SCD} algorithm solves this EOT problem via alternating optimization of particles and coupling:
\begin{subequations}
\label{eq:scd}
\begin{align}
    \seq{z}^{(k+1)} &\gets \argmin_{\seq{z}} \oper{L}^c_{\varepsilon}(\seq{z}, \Gamma^{(k)}; \seq{y}),\label{eq:scd particles}\\
    \Gamma^{(k+1)} &\gets \argmin_{\Gamma \in \set{\Gamma}^{(k+1)}} \oper{L}^c_{\varepsilon}(\seq{z}^{(k+1)}, \Gamma; \seq{y}), \label{eq:scd coupling}
\end{align}
\end{subequations}
where $\set{\Gamma}^{(k+1)} \define \set{\Gamma}(q(\seq{z}^{(k+1)}), p(\seq{y}))$ is the coupling constraint induced by the current particles. The coupling update \cref{eq:scd coupling} is solved efficiently via \Cref{alg:sinkhorn}. The particle update \cref{eq:scd particles} generally requires first-order optimization, with gradients available via the envelope theorem~\cite{Milgrom02, Peyre19}. However, for quadratic costs the solution is closed-form, e.g:

\begin{proposition}[Barycentric Update]
When $c(z, y) = \|z - y\|_2^2$, the minimizer of $\seq{z} \mapsto \oper{L}^c_{\varepsilon}(\seq{z}, \Gamma; \seq{y})$ is the barycentric projection:
\begin{align}
    \opt{z}_i = \sum_{j=1}^M \Gamma_{ij} y_j \bigg/ \sum_{k=1}^M \Gamma_{ik}.\label{eq:barycentric}
\end{align}
\end{proposition}

Appendix  \ref{appendix:b} establishes general conditions for closed-form updates and lists additional cases useful in robotics, e.g., spline parameters and elements of $\set{SE}(3)$.

\begin{algorithm}[t]
\caption{Sinkhorn Coordinate Descent (\algo{SCD})}
\label{alg:scd}
\KwIn{Init. particles $\seq{z}^{(0)}$, proposals $\seq{y}$, weights $q \in \set{\Delta}^N$, $p \in \set{\Delta}^M$, parameters $\varepsilon > 0$, $\eta \in (0, 1]$.}
\KwOut{Updated particles $\seq{z}$.}
\BlankLine
\Repeat{Convergence}{%
$\seq{y} \gets$ Resample from $R$ in parallel (Optional)\;
$C_{ij} \gets \frac{1}{2}\|z_i^{(k)} - y_j\|^2$\;
$\Gamma \gets \algo{Sink}(C, q(\seq{z}^{(k)}), p(\seq{y}), \varepsilon)$\tcp*[l]{\cref{eq:scd coupling}}
\For{$i = 1, \ldots, N$}{
    $b_i \gets \sum_{j=1}^M \Gamma_{ij} y_j \big/ \sum_{k=1}^M \Gamma_{ik}$\tcp*[l]{\cref{eq:barycentric}}
    $z_i^{(k+1)} \gets (1 - \eta)\, z_i^{(k)} + \eta\, b_i$\tcp*[l]{\cref{eq:scd particles}}
}
}
\end{algorithm}

% \begin{algorithm}[t]
% \caption{Sinkhorn Coordinate Descent (\algo{SCD})}
% \label{alg:scd}
% \KwIn{Init. particles $\seq{x}^{(0)}$, proposals $\seq{y}$, weights $q \in \set{\Delta}^N$, $p \in \set{\Delta}^M$, parameters $\varepsilon > 0, \eta \in (0, 1]$.}
% \KwOut{Updated particles $\seq{x}$.}
% \BlankLine
% \Repeat{Convergence}{%
% $\seq{y} \gets$ Resample $\seq{y}$ from $R(\seq{y}| \seq{z}^{(k)})$ (Optional)\;
% $C_{ij} \gets \frac{1}{2}\|z_i^{(k)} - y_j\|^2$\;
% $\Gamma \gets \algo{Sink}(C, q(\seq{z}^{(k)}), p(\seq{y}^{(k)}), \varepsilon)$\tcp*[l]{Eq.\Cref{eq:scd coupling}}
% \For{$i = 1, \ldots, N$}{
%     $b_i \gets \sum_{j = 1}^M\Gamma_{ij} y_{j} \big/ \sum_{k = 1}^M \Gamma_{ik}$\tcp*[r]{Eq.\Cref{eq:barycentric}}
%     $z_i^{(k+1)} \gets (1 - \eta)\, z_i^{(k)} + \eta\, b_i$\tcp*[r]{Eq.\Cref{eq:scd particles}}
% }
% }
% \end{algorithm}

% \begin{figure}[!b]
% \centering
% \includegraphics[width=\columnwidth]{images/barycenter_cartoon.png}%
%   \caption{\textbf{Barycentric projection.} Candidates (squares) couple to proposals (circles) with strength indicated by line thickness. Each candidate updates toward the weighted centroid of its coupled proposals.}%
%   \label{fig:barycentric}%
% \end{figure}

\textit{\textbf{Properties and Behavior.}} The complete \algo{SCD} algorithm is listed in \Cref{alg:scd}. In practice, the particle update \cref{eq:scd particles} is relaxed with step size $\eta \in (0, 1]$ to improve stability and reduce variance when proposals are resampled each iteration or $\varepsilon$ is very small.

Notably, \algo{SCD} generalizes the importance sampling procedure underlying \algo{MPPI}. When $N = 1$, the marginal constraint forces $\Gamma_{1j} = p_j$, and the barycenter \cref{eq:barycentric} reduces to the global weighted average in \cref{eq:mppi is}. Similarly, as $\varepsilon \to \infty$, the coupling factors as $\Gamma_{ij} = q_i p_j$, which causes the barycenter update in \cref{eq:barycentric} to be independent of the particle identity---causing all particles to have the same barycenter.  In both limits, the barycentric projection becomes an importance-weighted average and \algo{SCD} recovers \algo{MPPI}.

The algorithm enjoys favorable theoretical properties when proposals are fixed. The most important of which are stated concisely in the following proposition. Full details, additional theoretical results, and proofs in the Appendix  \ref{appendix:a}.
\begin{proposition}
When $\seq{y}$ are fixed and $z \mapsto c(z, y)$ is convex:
\begin{enumerate}[leftmargin=*, itemsep=1pt, topsep=1pt, label=\roman*.]
    \item The objective $\oper{L}^c_{\varepsilon}(\seq{z}, \Gamma; \seq{y})$ is biconvex in $(\seq{z}, \Gamma)$.
    \item The sequence $\qty(\oper{L}^c_{\varepsilon}(\seq{z}^{(k)}, \Gamma^{(k)}; \seq{y}))_{k \geq 0}$ is non-increasing.
    \item The iterates $(z^{(k)}, \Gamma^{(k)})$ converge to a stationary point.
\end{enumerate}
\end{proposition}
\noindent The biconvexity of $\oper{L}^c_{\varepsilon}(\seq{z}, \Gamma; \seq{y})$ ensures each subproblem has a unique solution. As a descent method, each iteration maintains or improves the solution. The convergence guarantee ensures the algorithm terminates at a well-defined fixed point rather than oscillating. Together, these properties make \algo{SCD} well-suited for real-time MPC: the algorithm can be stopped after a fixed iteration budget with the \emph{guarantee that each iteration has improved the solution}.

{\itshape \bfseries Exploration-Exploitation Tradeoff.} The transport structure shapes algorithmic behavior in several ways. First, particles remain in the convex hull of proposals---exploration is limited by proposal coverage, making the choice of reference distribution $R(\seq{y}|\seq{z})$ important. Second, marginal constraints coordinate updates: each particle must distribute its mass $q_i$ across proposals, and the transport cost biases this allocation toward nearby proposals. This prevents particles from collapsing to a single mode when proposals cover multiple basins. However, if proposals themselves cluster in one basin---e.g., when resampling around already-converged particles---diversity can still collapse.

For these reasons, the choice of $R(\seq{y}|\seq{z})$ and the resampling frequency significantly shape \algo{SCD}'s behavior. Sampling proposals near each particle enables local refinement but reduces exploration, potentially causing all particles to converge to the best known local minimum. Depending on the application, this can be desirable---especially when combined with annealing---since committing to the best discovered mode after sufficient exploration is a reasonable strategy when gradient information is unavailable.

\textbf{\textit{Stopping Criteria.}} When early stopping is preferred, some options for practical convergence criteria include,
\begin{align*}
    &\textit{Particle Displacement:} \hfill & \max_i \|z_i^{(k+1)} - z_i^{(k)}\| &< \delta,&\\
    &\textit{Relative Improvement:} \hfill & (\oper{L}^{(k)} - \oper{L}^{(k+1)})/\oper{L}^{(k)} &< \delta.&
\end{align*}
\begin{figure*}[!t]
    \centering
    \includegraphics[width=1\linewidth]{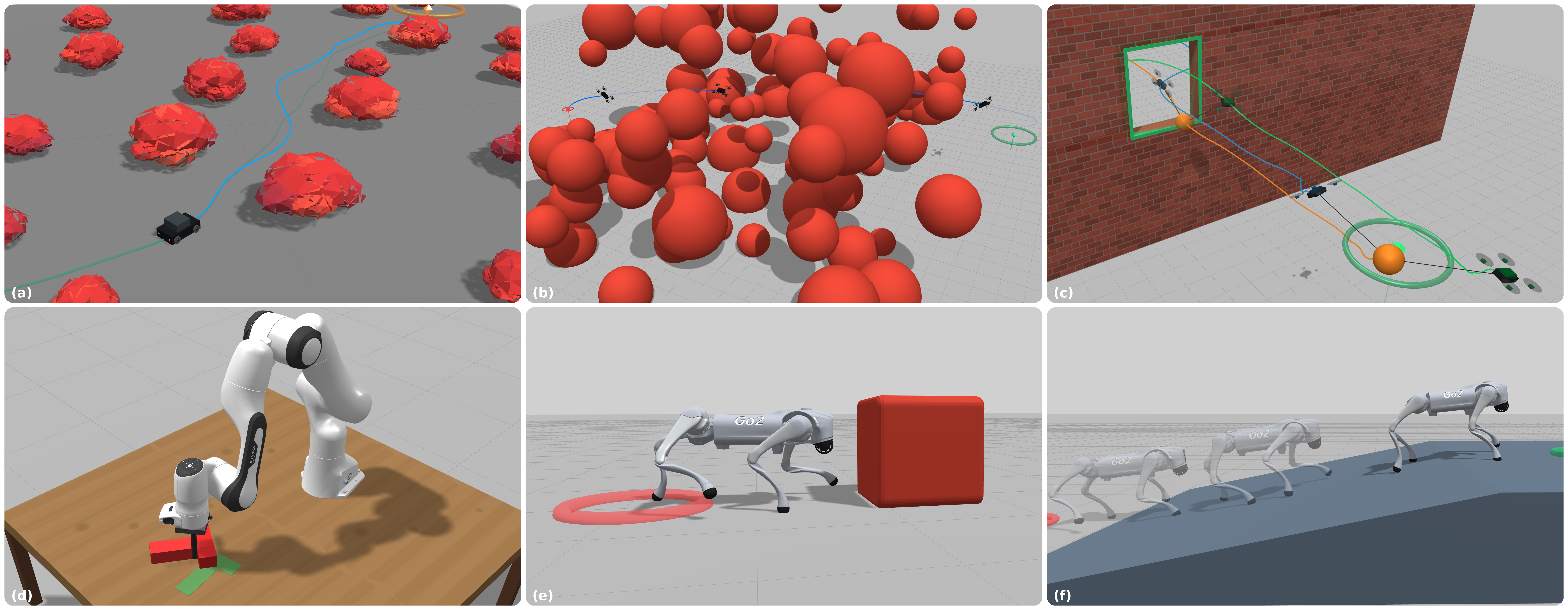 }
    \caption{Robotics control experiments used to evaluate \algo{OT-MPC}. (a) Kinematic Bicycle navigating through an obstacle field. (b) Quadrotor navigating from start to goal through cluttered environment. (c) Two-Quadrotor system cooperatively carrying a suspended load through an opening in the wall. (d) Push-T task using Franka where the manipulator has to push and align the T-Block to a goal location (e) Unitree Go2 box pushing task, where the quadruped has to push the box to a goal location (f) Unitree Go2 locomotion task of climbing an inclined ramp. }
    \label{fig:experiments}
\end{figure*}

\section{Model-Predictive Control via Entropic Optimal Transport}
\label{sec:ot-mpc}
% Section 4: Sinkhorn MPC
% Include with \input{sections/sinkhorn_mpc}

This section instantiates \algo{SCD} for control, creating \emph{Optimal Transport MPC} (\algo{OT-MPC}). We specify the particle representation, proposal distribution, and computational trade-offs.

\subsection{Trajectory Optimization via Sinkhorn Coordinate Descent}
\label{ssec:traj opt}

Each particle $z_i = (u_i^{0}, u_i^{1}, \ldots, u_i^{t_f-1})$ represents a candidate control sequence. The proposal weights follow the Gibbs distribution $p_j \propto \exp(-\beta S(y_j; x_0))$, where $\beta > 0$ is an inverse temperature controlling concentration at low-cost proposals. Particles receive uniform weights $q_i = 1/N$ to encourage exploration. At each MPC cycle, we run $K$ iterations of \algo{SCD}, drawing fresh proposals each iteration, then execute the first control from the lowest-cost particle. Standard warm-starting applies: particles are shifted forward in time and the final segment is reinitialized.

\subsection{Proposal Distribution}
\label{ssec:proposals}

The proposal distribution $R(\seq{y} | \seq{z})$ balances local refinement against global exploration. We use a mixture to sample proposals:
\begin{align}
    R(y|\seq{z}) = (1 - \rho) \cdot \frac{1}{N} \sum_{i=1}^{N} \mathcal{N}(y; z_i, \Sigma) + \rho \cdot R_{\mathrm{global}}(y), \label{eq:proposal mixture}
\end{align}
where $\rho \in [0, 1]$ controls the exploration rate. The first component perturbs existing particles, enabling local refinement around promising solutions. The second component $R_{\mathrm{global}}$ provides global coverage---either uniform over the control bounds or a broad Gaussian centered at zero.

The perturbation covariance $\Sigma$ can be isotropic ($\sigma^2 I$) or structured to reflect problem geometry. For trajectory optimization, temporal correlations often improve sample quality: perturbations that vary smoothly across timesteps produce dynamically coherent candidates, whereas independent noise at each timestep yields low-quality erratic trajectories.

\subsection{Hyperparameter Selection}
\label{ssec:hyperparameters}

\algo{OT-MPC} introduces three key hyperparameters beyond those shared with \mppi{}: the entropy regularization $\varepsilon$, the relaxation parameter $\eta$, and the number of particles $N$.

The regularization $\varepsilon$ controls coupling sparsity. Small $\varepsilon$ yields near-deterministic assignment where each particle couples primarily to its nearest low-cost proposal; this accelerates convergence but risks premature commitment. Large $\varepsilon$ spreads coupling mass broadly, maintaining diversity but slowing refinement. We find $\varepsilon$ in the range $[0.01, 0.1]$ times the median pairwise distance works well across tasks.

The relaxation parameter $\eta \in (0, 1]$ governs step size toward the barycenter. Full steps ($\eta = 1$) converge fastest when proposals are fixed, but cause oscillation when proposals are resampled each iteration. Damped updates ($\eta \approx 0.5$) provide stability at the cost of slower convergence. In practice, $\eta$ between $0.3$ and $0.7$ balances these concerns.

The particle count $N$ determines the capacity to represent multimodal structure. Too few particles collapse to a single mode; too many incur unnecessary coupling cost. We observe diminishing returns beyond $N \approx 10$--$20$ for problems with two to four distinct modes. Since the coupling cost scales as $\mathcal{O}(NM)$, a practical heuristic is to set $N \ll M$, using many proposals for exploration but few particles to track the discovered modes.

\section{Experiments}
\label{sec:experiments}

We evaluate \algo{OT-MPC} across a diverse set of robotics control tasks ranging from navigation, locomotion to manipulation. The experiments are designed to test and evaluate whether the optimal transport coupling improves the performance on tasks with multimodal cost landscapes that generally cause other sampling based methods to struggle. 

\textbf{\textit{Baselines.}} We benchmark against \algo{MPPI}, \algo{CEM}  and Stein Variational Model-Predictive Control (\algo{SV-MPC}) on the lower-dimensional tasks such as bicycle navigation and planar Push-T task. We decided to exclude \algo{CEM}  and \algo{SV-MPC}  from higher dimensional system for the following reasons. As the state dimensions increase the \algo{CEM}'s elite selection mechanism performs poorly as it discards majority of the cost information and fails to capture complex cost landscapes. \algo{SV-MPC}, which builds on \algo{SVGD} requires differentiable and smooth dynamics and cost functions which may not be feasible for complex and larger dimensional systems that includes rich contact dynamics and sparse cost structure (e.g. indicator functions for collision). As \algo{OT-MPC} and \algo{MPPI}  are both zeroth order sampling based methods, \algo{OT-MPC} can be fairly compared against \algo{MPPI}.

%\textit{\textbf{MPPI  Variants.}} Although, there are numerous variants of \algo{MPPI}  including colored noise \algo{MPPI}  \cite{Vlahov24}, log-space formulation \cite{mohamed2022autonomous} and recent improvements using annealing schemes \cite{Xue24, schramm2025reference}, we focus on a fair comparison with the vanilla \algo{MPPI}. This is because these variants do not generally change the underlying \algo{MPPI}  update mechanism, and most of these enhancements can be directly adapted to \algo{OT-MPC} to confer the same benefits.

\textit{\textbf{MPPI  Variants.}}  OT-MPC generalizes MPPI in the cases $N = 1$ or $\varepsilon \to \infty$ because SCD recovers the MPPI importance-weighted average exactly (Section~V). The two methods share the same high-level structure---sample proposals, roll out dynamics, evaluate costs, update candidate controls. They differ only in how the candidates integrate sample information: MPPI uses a global weighted average while OT-MPC uses optimal transport. Comparing against vanilla MPPI therefore \emph{isolates the contribution of the OT-based update rule}. Most enhancements to MPPI, e.g., colored noise \cite{Vlahov24}, annealing schedules \cite{Xue24, schramm2025reference}, and log-space formulations \cite{mohamed2022autonomous}, improve proposal generation and scoring but do not modify the aggregation step, and thus \emph{transfer directly to OT-MPC without modification}.

\textbf{\textit{Hyperparameter Tuning.}} To ensure fair comparison, the hyperparameters for all control algorithm are tuned using Optuna \cite{akiba2019optuna} with the same hyperparameter tuning objective and whenever possible with the same sampling budget. Complete details on \emph{cost structures, hyperparameter values, implementation details and additional visualization of successes and failures} are provided in the Appendix. Videos and an interactive 3D visualizer for qualitative inspection of trajectories are available on the project page.

\subsection{Car (Bicycle) Obstacle Avoidance}

We first evaluate \algo{OT-MPC} on a car navigation task using a kinematic bicycle model with states $(x,y, \theta, v)$ representing the position, heading and velocity of the system and with linear acceleration and steering angle as control inputs. The objective is to navigate the car from a starting location to a goal location through a dense obstacle field avoiding collision. We benchmark \algo{OT-MPC} against \algo{MPPI}, \algo{CEM}, and \algo{SV-MPC} across 300 Monte Carlo trials with randomly generated initializations, goal locations and obstacle field at two difficulty levels: \textit{Easy} (sparser obstacles) and \textit{Hard} (denser obstacle field). The hyperparameters of all the controllers were tuned using Optuna for a fair comparison. Table \ref{tab:benchmark} shows the benchmark results and we can see that \algo{OT-MPC} achieves the highest success rate (99\% on Easy and 93.5\% on Hard), outperforming \algo{MPPI}  and other control algorithms. The performance gap widens in the \textit{Hard} settings where the denser obstacles create more multimodal cost landscapes and control schemes like \algo{MPPI}  struggle due to mode-averaging. Other details regarding the parameters, cost structure and breakdown of the benchmark results can be found in the Appendix  \ref{appendix:c}.

\subsection{2D Push-T}

The planar Push-T task is a well-known manipulator task used to evaluate controllers in a multimodal scenario \cite{Chi25}. The objective here is to push and align a T-shaped block from a randomly perturbed initial location to a goal location using a circular pusher. This is an inherently difficult task for standard MPC and sampling based controllers due to the hybrid contact dynamics and sparse costs. Here the states of the system are the pusher position as well as the position and orientation of the T-block and we use velocity control to move the pusher. We benchmark our \algo{OT-MPC} with other controllers across 50 randomly generated initial and goal configuration for the T-block. The results of the benchmark are provided in Table \ref{tab:benchmark}. We can clearly observe the superior performance of \algo{OT-MPC} (76\% success) compared to the other control algorithms. Again here \algo{MPPI}  struggles due to mode-averaging in this contact-rich settings where there might be multiple viable solutions. \algo{SV-MPC}  was not able to solve this problem, due to the lack of meaningful gradients from the hybrid contact dynamics.

\subsection{Quadrotor Dense Obstacle Avoidance}

In order to evaluate \algo{OT-MPC} in higher dimensional systems, we have considered the 12DOF Quadrotor with thrust-torque control navigating in a dense obstacle  field from a starting location to a goal avoiding collision. We tested our \algo{OT-MPC} controller against \algo{MPPI}  across 100 trials each on three different environment difficulty settings - \textit{Easy} (50 obstacles), \textit{Medium} (100 obstacles), \textit{Hard} (100 obstacles in tighter configuration). As shown in Table \ref{tab:benchmark} \algo{OT-MPC} outperform \algo{MPPI}  in both \textit{Medium} (100\% vs 60\%) and \textit{Hard} (92\% vs 19\%) settings. We observed that the low success rate in \algo{MPPI}  is not due to collision with obstacle but rather due to \algo{MPPI} failing to find a feasible path in this dense obstacle field and often gets stuck in local minima while OT-MPC's multimodal approach helps alleviate this issue.

\subsection{Two Quadrotor Cooperative Load Carrying}

\begin{table}[!t]
    \footnotesize
    \centering
    \begin{tabular}{lc c c  }
    \toprule
        \textbf{Task} & \textbf{Controller} & \textbf{Success (\%)} & \textbf{Avg. Steps}  \\
    \midrule
        \multirow{4}{*}{\makecell[l]{Car Obstacle \\ \textit{Easy}}}
        & \algo{OT-MPC} & \textbf{99.00} & $76.1 \pm 21.6$  \\
        & \algo{MPPI} & 95.00 & $\mathbf{57.2 \pm 5.7}$  \\
        & \algo{CEM} & 70.00 & $79.5 \pm 15.2$  \\
        & \algo{SV-MPC} & 93.00 & $84.8 \pm 22.4$  \\
    \midrule
        \multirow{4}{*}{\makecell[l]{Car Obstacle \\ \textit{Hard}}}
        & \algo{OT-MPC} & \textbf{93.50} & $89.5 \pm 33.8$  \\
        & \algo{MPPI} & 88.50 & $\mathbf{61.7 \pm 24.4}$  \\
        & \algo{CEM} & 56.00 & $82.8 \pm 14.5$  \\
        & \algo{SV-MPC} & 78.50 & $105.0 \pm 46.2$  \\
    \midrule
    \multirow{4}{*}{2D Push-T}
        & \algo{OT-MPC} & \textbf{76.00} & $\mathbf{33.08 \pm 11.35}$  \\
        & \algo{MPPI} & 4.00 & $83.0 \pm 1.0 ^\dagger$   \\
        & \algo{CEM} & 46.00 & $47.26 \pm 19.95$  \\
        & \algo{SV-MPC} & 0.00 & ---  \\
    \midrule
        \multirow{2}{*}{\makecell[l]{Franka Push-T}}
        & \algo{OT-MPC} & \textbf{66.00} & $52.67 \pm 35.03$  \\
        & \algo{MPPI} & 64.00 & $\mathbf{46.69 \pm 17.87}$  \\
    \midrule
        \multirow{2}{*}{\makecell[l]{Quad. Obstacle \\ \textit{Easy}}}
        & \algo{OT-MPC} & \textbf{100.00} & $119.7 \pm 19.9$  \\
        & \algo{MPPI} & 100.00 & $\mathbf{93.6 \pm 9.3}$  \\
    \midrule
        \multirow{2}{*}{\makecell[l]{Quad. Obstacle \\ \textit{Medium}}}
        & \algo{OT-MPC} & \textbf{100.00} & $159.6 \pm 33.3$  \\
        & \algo{MPPI} & 60.00 & $\mathbf{127.4 \pm 23.3}$  \\
    \midrule
        \multirow{2}{*}{\makecell[l]{Quad. Obstacle \\ \textit{Hard}}}
        & \algo{OT-MPC} & \textbf{92.00} & $217.8 \pm 56.6$  \\
        & \algo{MPPI} & 19.00 & $\mathbf{147.4 \pm 13.6}$  \\
    \midrule
    \multirow{2}{*}{\makecell[l]{Quad. Carry \\ \textit{Normal}}}
        & \algo{OT-MPC} & \textbf{91.00} & $\mathbf{143.93 \pm 27.28}$  \\
        & \algo{MPPI} & 22.00 & $371.45 \pm 44.65$  \\
    \midrule
    \multirow{2}{*}{\makecell[l]{Quad. Carry \\ \textit{Hard}}}
        & \algo{OT-MPC} & \textbf{75.00} & $\mathbf{202.13 \pm 89.75}$  \\
        & \algo{MPPI} & 10.00 & $369.2 \pm 42.22$  \\
    \bottomrule
    \end{tabular}
    \caption{Benchmark results comparing \algo{OT-MPC} (ours) with various control schemes on a variety of robotics control tasks. The average steps for each task is computed using \emph{only successful runs}. $\dagger$ The low standard deviation is due to the fact that \algo{MPPI} only completed two runs successfully.}
    \label{tab:benchmark}
\end{table}

We have extended our quadrotor experiment to two-quadrotor system trying to cooperatively carry a suspended load from an initial location to a goal. Here we have a 27DOF system (Two 12DOF quadrotors and 3D position of the suspended load) with a 6-dimensional control - thrust-torque control for each quadrotor but with yaw torque control disabled due to cable constraints.  To increase the complexity of the experiment, we have designed the environment in which the two-quadrotor system must navigate through an opening in the wall to reach the other side (goal). We have benchmark \algo{MPPI}  and \algo{OT-MPC} in two different difficulty setting with locations of holes and initial states of the system randomized - \textit{Normal} (2.0m x 2.0m opening) and \textit{Hard} (1.2m x 1.4m opening). The benchmark result from 100 runs (each) are shown in Table \ref{tab:benchmark} and we observe that \algo{OT-MPC} performs exceptionally well in both \textit{Normal} (91 \% vs 22\%) and \textit{Hard} (75\%, 10\%) when compared to \algo{MPPI} . Here the narrow opening creates a bottleneck where both quadrotors need to coordinate precisely to navigate through and \algo{MPPI} 's mode-averaging disrupts this coordination.

\subsection{Franka Push-T}

We extended the planar Push-T task to full 3D manipulation setting using a Franka Panda arm and end-effector poking stick. The Franka state is 14-dimensional (joint positions and velocities) with 7-dimensional joint position control. Along with the T-block states (13 dimensional - position, quaternion, velocity and angular velocity) and End-Effector states (7-dimensional - End-effector position and quaternion), the total system is 34-dimensional. We parameterize the control trajectories using cubic splines to produce smooth joint motion. The objective here is to push a T-Block from an initial pose to a goal using the end-effector stick. 
We benchmark over 50 trials with randomized T-Block configurations. MPPI initially showed low success rate on this task, requiring significantly more hyperparameter tuning effort compared to OT-MPC to achieve competitive performance. After extensive tuning, as shown in Table \ref{tab:benchmark} we were able to reach similar success rate (66\% for OT-MPC vs 64\% for MPPI). This result shows that OT-MPC can achieve good performance with less tuning in complex tasks and can probably outperform MPPI, if similar tuning effort is used.

\subsection{Quadruped Locomotion}

We also evaluated the effectiveness of our controller on a simulated Unitree Go2 quadruped on locomotion task in different terrains. The Go2 quadruped is 37-dimensional system where the control inputs are the joint position of 12 leg joints. We employ a cubic hermite spline parameterization \cite{schramm2025reference} to sample smooth joint positions and velocities that enables us to produce coherent motions and the gaits emerge naturally from the sampling-based optimization rather than requiring reference trajectories or predefined gaits. We tuned our cost function weights for both \algo{MPPI}  and \algo{OT-MPC} in flat terrain to produce stable walking gait. To test generalization capability, we then evaluate these algorithms on out of distribution environments like inclined ramp (mild inclination) and narrow bridge crossing.
The results are shown in Table \ref{tab:locomotion}. Since here we want to evaluate the locomotion performance, we have decided to report the average steps and median distance to goal across 5 terrain configurations. Both \algo{MPPI}  and \algo{OT-MPC} achieve stable locomotion, however, \algo{OT-MPC} achieves significantly lower average steps compared to \algo{MPPI}  indicating superior locomotion performance. 

\subsection{Quadruped Box Pushing}

We have also extended our Quadruped locomotion task to a contact rich loco-manipulation task where the quadruped is tasked with pushing a box from a starting location to a goal location. Here with the addition of the box, the state space expands to 37 robot states and 13 box states (position, quaternions, velocity and angular velocities). We benchmark the performance of \algo{MPPI}  and \algo{OT-MPC} across 100 trials with randomized initial robot pose and box initializations (based on difficulty. See Appendix  \ref{appendix:c} for details).  Table \ref{tab:locomotion} shows the benchmark results and we can see that \algo{OT-MPC} performs better while maintaining lower average steps and median distance to goal.

\subsection{Computational Overhead}

For tasks with JAX-vectorized rollouts, Table~\ref{tab:compute_time} reports per-iteration and wall-clock times. OT-MPC achieves comparable or faster wall-clock times on most tasks despite the additional Sinkhorn step. This is because Optuna consistently selects fewer proposals per OT-MPC iteration. The transport coupling extracts more information per sample, reducing the number of rollouts needed. Since rollouts dominate compute, this efficiency more than compensates for the Sinkhorn overhead.
 
For the Franka and Go2 experiments, sequential MuJoCo rollouts dominate wall-clock time, making overall timings unrepresentative of algorithmic cost. To isolate the OT-MPC overhead, we measured the per-iteration cost of the Sinkhorn solve and barycentric update: 1.6\,ms for Franka ($8 \times 800$ coupling) and 0.2\,ms for Go2 ($8 \times 50$).

\begin{table}[!t]
    \centering
    \begin{tabular}{l c r@{${}\pm{}$}l c}
    \toprule
        \textbf{Task} & \textbf{Controller} & \multicolumn{2}{c}{\textbf{Avg. Steps}} & \parbox{2cm}{\centering \textbf{Median Goal Dist.}}\\
    \midrule
    \multirow{2}{*}{Locomotion}
        & \algo{OT-MPC} & $\textbf{427.8}$ & $\textbf{89.06}$ & \textbf{0.205}\\
        & \algo{MPPI} & $799.2$ & $150.72$ & 0.249 \\
    \midrule
    \multirow{2}{*}{Box Pushing}
        & \algo{OT-MPC} & $\textbf{101.0}$ & $\textbf{23.0}$ & \textbf{0.084} \\
        & \algo{MPPI} & $124.9$ & $65.4$ & 0.16\\
    \bottomrule
    \end{tabular}
    \caption{Results of Go2 Quadruped locomotion tasks. In the locomotion task both control methods achieved 100\% success when tested in 5 different terrain configurations. In the box pushing task, \algo{OT-MPC} achieved a 78\% success while \algo{MPPI} achieved 71\% across 100 Monte Carlo runs.}
    \label{tab:locomotion}
\end{table}

\begin{table}[t]
    \centering
    \footnotesize
    \begin{tabular}{l c c  c  c c }
    \toprule
      \multirow{2}{*}{Environment} & \multicolumn{2}{c}{Time / Iter. (ms)} & \multicolumn{2}{c}{Wall Clock (ms)} \\
        \cmidrule(lr){2-3} \cmidrule(lr){4-5}
         & MPPI & OT MPC & MPPI & OT MPC \\
        \midrule
        Car & 3.07 & 4.73 & 24.57 & 27.81   \\
        2D Push-T & 1.44 & 0.95 & 5.75 & 15.12 \\
        Quadrotor  & 67.6 & 26.6 & 338.0 & 133.0 \\
        Two Quad Carry & 23.1 & 12.4 & 230.7 & 99.1 \\
        \bottomrule
    \end{tabular}
    \caption{\footnotesize Computation times for JAX-vectorized environments. Each method uses its Optuna-tuned sample count (e.g., Quadrotor: MPPI uses 2000 proposals; OT-MPC uses 200).}
    \label{tab:compute_time}
    \vspace{-2.2em}
\end{table}

\section{Conclusion}
\label{sec:conclusion}
This paper addresses the fundamental limitations of existing sampling-based optimal control algorithms, e.g., \algo{MPPI} and \algo{CEM}, caused by their information-theoretic foundations. The variational principles from which they are derived cannot incorporate geometric information, resulting in solutions that blend distinct modes or commit prematurely to one. To ameliorate these limitations, we derive a novel sampling algorithm, Sinkhorn Coordinate Descent (\algo{SCD}), founded in optimal transport rather than information theory, and instantiate it in a model-predictive control scheme (\algo{OT-MPC}). The optimal transport cost structure incorporates geometric proximity---enabling both local refinement and mode preservation behaviors not found in existing methods.

Moreover, we established theoretical properties of \algo{SCD} including monotone descent and convergence guarantees, and demonstrated empirically that \algo{OT-MPC} outperforms \algo{MPPI}, \algo{CEM}, and \algo{SV-MPC} across navigation, manipulation, and locomotion tasks. The performance gap is most pronounced in settings with multimodal cost landscapes---e.g., dense obstacle fields, contact-rich manipulation, and coordinated multi-robot control---where existing methods struggle.

\textit{\textbf{Limitations.}} \algo{OT-MPC} incurs additional computational cost from Sinkhorn iterations compared to \algo{MPPI}. Its complexity scales as $\oper{O}(NM)$ per iteration compared to the $\oper{O}(M)$ for an \algo{MPPI} update---however, \algo{OT-MPC} did not incur a significant overhead for the quantities of particles used in our experiments ($N \approx 10$--$20$). The algorithm's performance depends on the entropy regularization $\varepsilon$. Small values risk premature commitment while excessively large values slow refinement. We provided practical guidelines, but adaptive scheduling remains an open question. Finally, because particles update toward barycenters of proposals, exploration is fundamentally limited by proposal coverage---poor initialization or overly local sampling can still cause diversity collapse.

\textit{\textbf{Future Work.}} There are a number of exciting avenues for future work. As a novel sampling algorithm, \algo{SCD} has broad application beyond \algo{OT-MPC} and would benefit from detailed comparisons to other sampling methods to identify its strengths and weaknesses across problem classes. There are also various algorithmic extensions to explore. For example, there are opportunities to dynamically adapt the proposal distribution by extending particles from points to Gaussian distributions or leveraging natural connections between OT and Laguerre tessellations to create a semi-discrete algorithm. Finally, the theoretical results we established for \algo{SCD} are for the setting where proposals are fixed---generalizing them to situations where proposals are resampled can better inform our understanding of this algorithm.

\bibliographystyle{unsrtnat}
\bibliography{references}

@book{Villani08,
  title={Optimal transport: old and new},
  author={Villani, C{\'e}dric},
  volume={338},
  year={2008},
  publisher={Springer}
}

@inproceedings{Cuturi13,
  title={Sinkhorn distances: Lightspeed computation of optimal transport},
  author={Cuturi, Marco},
  booktitle={Advances in Neural Information Processing Systems},
  volume={26},
  year={2013}
}

@article{Peyre19,
  title={Computational optimal transport: With applications to data science},
  author={Peyr{\'e}, Gabriel and Cuturi, Marco},
  journal={Foundations and Trends in Machine Learning},
  volume={11},
  number={5-6},
  pages={355--607},
  year={2019}
}

@article{Williams17,
  title={Model Predictive Path Integral Control: From theory to parallel computation},
  author={Williams, Grady and Wagener, Nolan and Goldfain, Brian and Drews, Paul and Rehg, James M. and Boots, Byron and Theodorou, Evangelos A.},
  journal={Journal of Guidance, Control, and Dynamics},
  volume={40},
  number={2},
  pages={344--357},
  year={2017}
}

@article{Rubinstein99,
  title={The cross-entropy method for combinatorial and continuous optimization},
  author={Rubinstein, Reuven Y.},
  journal={Methodology and Computing in Applied Probability},
  volume={1},
  number={2},
  pages={127--190},
  year={1999}
}

@article{Williams18,
  title={Information-theoretic {MPC} for model-based reinforcement learning},
  author={Williams, Grady and Wagener, Nolan and Goldfain, Brian and Drews, Paul and Rehg, James M. and Boots, Byron and Theodorou, Evangelos A.},
  journal={International Conference on Robotics and Automation},
  pages={1714--1721},
  year={2017}
}

@inproceedings{Yi24,
  title={{CoVO-MPC}: Theoretical analysis of sampling-based {MPC} and optimal covariance design},
  author={Yi, Zeji and Pan, Chaoyi and He, Guanqi and Qu, Guannan and Shi, Guanya},
  booktitle={Learning for Dynamics and Control Conference},
  pages={1122--1135},
  organization={PMLR},
  year={2024}
}

@article{Xue24,
  title={Full-order sampling-based {MPC} for torque-level locomotion control via diffusion-style annealing},
  author={Xue, Haoru and Pan, Chaoyi and Yi, Zeji and Qu, Guannan and Shi, Guanya},
  journal={arXiv preprint arXiv:2409.15610},
  year={2024}
}

@inproceedings{Kobilarov11,
  title={Cross-entropy randomized motion planning},
  author={Kobilarov, Marin},
  booktitle={Robotics: Science and Systems},
  year={2011}
}

@inproceedings{Pinneri20,
  title={Sample-efficient cross-entropy method for real-time planning},
  author={Pinneri, Cristina and Sawant, Shambhuraj and Blaes, Sebastian and Achterhold, Jan and Stueckler, Joerg and Rolinek, Michal and Martius, Georg},
  booktitle={Conference on Robot Learning},
  pages={1049--1065},
  organization={PMLR},
  year={2020}
}

@inproceedings{Wang21,
  title={Variational inference {MPC} using {T}sallis divergence},
  author={Wang, Ziyi and So, Oswin and Gibson, Jason and Vlahov, Bogdan and Gandhi, Manan S and Liu, Guan-Horng and Theodorou, Evangelos A},
  booktitle={Robotics: Science and Systems},
  year={2021}
}

@article{Ito23,
  title={Entropic model predictive optimal transport over dynamical systems},
  author={Ito, Kaito and Kashima, Kenji},
  journal={Automatica},
  volume={152},
  pages={110980},
  year={2023}
}

@inproceedings{Le23,
  title={Accelerating motion planning via optimal transport},
  author={Le, An T and Chalvatzaki, Georgia and Biess, Armin and Peters, Jan},
  booktitle={Advances in Neural Information Processing Systems},
  volume={36},
  year={2023}
}

@inproceedings{Honda24,
  title={{S}tein variational guided model predictive path integral control: Proposal and experiments with fast maneuvering vehicles},
  author={Honda, Kohei and Akai, Naoki and Suzuki, Kosuke and Aoki, Mizuho and Hosogaya, Hirotaka and Okuda, Hiroyuki and Suzuki, Tatsuya},
  booktitle={IEEE International Conference on Robotics and Automation},
  pages={8604--8610},
  year={2024}
}

@inproceedings{Alvarez25,
  title={Real-time whole-body control of legged robots with model-predictive path integral control},
  author={Alvarez-Padilla, Juan and Zhang, John Z. and Kwok, Sofia and Dolan, John M. and Manchester, Zachary},
  booktitle={International Conference on Robotics and Automation},
  pages={14721--14727},
  year={2025},
  organization={IEEE}
}

@article{Abraham20,
  title={Model-based generalization under parameter uncertainty using path integral control},
  author={Abraham, Ian and Handa, Ankur and Ratliff, Nathan and Lowrey, Kendall and Murphey, Todd D. and Fox, Dieter},
  journal={IEEE Robotics and Automation Letters},
  volume={5},
  number={2},
  pages={2864--2871},
  year={2020},
  publisher={IEEE}
}

@inproceedings{Keshavarz25,
  title={Control of Legged Robots using Model Predictive Optimized Path Integral},
  author={Keshavarz, Hossein and Ramirez-Serrano, Alejandro and Khadiv, Majid},
  booktitle={International Conference on Humanoid Robots},
  pages={1--8},
  year={2025},
  organization={IEEE}
}

@inproceedings{Hansen24,
title={{TD}-{MPC}2: {S}calable, Robust World Models for Continuous Control},
author={Nicklas Hansen and Hao Su and Xiaolong Wang},
booktitle={The Twelfth International Conference on Learning Representations},
year={2024}
}

@inproceedings{Okada20,
  title={Variational Inference {MPC} for {B}ayesian Model-Based Reinforcement Learning},
  author={Okada, Masashi and Taniguchi, Tadahiro},
  booktitle={Conference on robot learning},
  pages={258--272},
  year={2020},
  organization={PMLR}
}

@inproceedings{Lambert21,
  title={Stein Variational Model Predictive Control},
  author={Lambert, Alexander and Ramos, Fabio and Boots, Byron and Fox, Dieter and Fishman, Adam},
  booktitle={Conference on Robot Learning},
  pages={1278--1297},
  year={2021},
  organization={PMLR}
}

@inproceedings{Attias03,
  title={Planning by Probabilistic Inference},
  author={Attias, Hagai},
  booktitle={International Workshop on Artificial Intelligence and Statistics},
  pages={9--16},
  year={2003},
  organization={PMLR}
}

@article{Vlahov24,
  title={Low frequency sampling in model predictive path integral control},
  author={Vlahov, Bogdan and Gibson, Jason and Fan, David D. and Spieler, Patrick and Agha-mohammadi, Ali-akbar and Theodorou, Evangelos A.},
  journal={Robotics and Automation Letters},
  volume={9},
  number={5},
  pages={4543--4550},
  year={2024},
  publisher={IEEE}
}

@inproceedings{Bhardwaj22,
  title={{STORM}: {A}n Integrated Framework for Fast Joint-Space Model-Predictive Control for Reactive Manipulation},
  author={Bhardwaj, Mohak and Sundaralingam, Balakumar and Mousavian, Arsalan and Ratliff, Nathan D. and Fox, Dieter and Ramos, Fabio and Boots, Byron},
  booktitle={Conference on Robot Learning},
  pages={750--759},
  year={2022},
  organization={PMLR}
}

@inproceedings{Aoyama24,
  title={Generalized Maximum Entropy Differential Dynamic Programming},
  author={Aoyama, Yuichiro and Theodorou, Evangelos A.},
  booktitle={Conference on Decision and Control},
  pages={8825--8831},
  year={2024},
  organization={IEEE}
}

@inproceedings{Ratheesh25,
  title={Operator splitting covariance steering for safe stochastic nonlinear control},
  author={Ratheesh, Akash and Pacelli, Vincent and Saravanos, Augustinos D. and Theodorou, Evangelos A.},
  booktitle={Conference on Decision and Control},
  pages={3552--3559},
  year={2025},
  organization={IEEE}
}

@article{Okamoto19,
  title={Optimal Stochastic Vehicle Path Planning using Covariance Steering},
  author={Okamoto, Kazuhide and Tsiotras, Panagiotis},
  journal={Robotics and Automation Letters},
  volume={4},
  number={3},
  pages={2276--2281},
  year={2019},
  publisher={IEEE}
}

@inproceedings{Saravanos24,
  title={Distributed Model Predictive Covariance Steering},
  author={Saravanos, Augustinos D. and Balci, Isin M. and Bakolas, Efstathios and Theodorou, Evangelos A.},
  booktitle={International Conference on Intelligent Robots and Systems},
  pages={5740--5747},
  year={2024},
  organization={IEEE}
}

@inproceedings{Yin22,
  title={Trajectory Distribution Control for Model Predictive Path Integral Control Using Covariance Steering},
  author={Yin, Ji and Zhang, Zhiyuan and Theodorou, Evangelos and Tsiotras, Panagiotis},
  booktitle={International Conference on Robotics and Automation},
  pages={1478--1484},
  year={2022},
  organization={IEEE}
}

@article{Chi25,
  title={Diffusion Policy: {V}isuomotor policy learning via action diffusion},
  author={Chi, Cheng and Xu, Zhenjia and Feng, Siyuan and Cousineau, Eric and Du, Yilun and Burchfiel, Benjamin and Tedrake, Russ and Song, Shuran},
  journal={International Journal of Robotics Research},
  volume={44},
  number={10-11},
  pages={1684--1704},
  year={2025},
  publisher={Sage Publications}
}

@article{Ho20,
  title={Denoising Diffusion Probabilistic Models},
  author={Ho, Jonathan and Jain, Ajay and Abbeel, Pieter},
  journal={Advances in Neural Information Processing Systems},
  volume={33},
  pages={6840--6851},
  year={2020}
}

@inproceedings{Dadashi21,
  title={Primal {W}asserstein Imitation Learning},
  author={Dadashi, Robert and Hussenot, L{\'e}onard and Geist, Matthieu and Pietquin, Olivier},
  booktitle={International Conference on Learning Representations},
  year={2021}
}

@inproceedings{Papagiannis22,
  title={Imitation Learning with Sinkhorn Distances},
  author={Papagiannis, Georgios and Li, Yunpeng},
  booktitle={Joint European Conference on Machine Learning and Knowledge Discovery in Databases},
  pages={116--131},
  year={2022}
}

@inproceedings{Janner22,
  title={Planning with Diffusion for Flexible Behavior Synthesis},
  author={Janner, Michael and Du, Yilun and Tenenbaum, Joshua and Levine, Sergey},
  booktitle={International Conference on Machine Learning},
  pages={9902--9915},
  year={2022},
  organization={PMLR}
}

@article{Milgrom02,
  title={Envelope theorems for arbitrary choice sets},
  author={Milgrom, Paul and Segal, Ilya},
  journal={Econometrica},
  volume={70},
  number={2},
  pages={583--601},
  year={2002},
  publisher={Wiley}
}

@inproceedings{Han18,
  title={{S}tein variational gradient descent without gradient},
  author={Han, Jun and Liu, Qiang},
  booktitle={International Conference on Machine Learning},
  pages={1900--1908},
  year={2018},
  organization={PMLR}
}

@inproceedings{akiba2019optuna,
  title={Optuna: A next-generation hyperparameter optimization framework},
  author={Akiba, Takuya and Sano, Shotaro and Yanase, Toshihiko and Ohta, Takeru and Koyama, Masanori},
  booktitle={Proceedings of the 25th ACM SIGKDD international conference on knowledge discovery \& data mining},
  pages={2623--2631},
  year={2019}
}

@article{schramm2025reference,
  title={Reference-Free Sampling-Based Model Predictive Control},
  author={Schramm, Fabian and Fabre, Pierre and Perrin-Gilbert, Nicolas and Carpentier, Justin},
  journal={arXiv preprint arXiv:2511.19204},
  year={2025}
}

@article{mohamed2022autonomous,
  title={Autonomous navigation of agvs in unknown cluttered environments: log-mppi control strategy},
  author={Mohamed, Ihab S and Yin, Kai and Liu, Lantao},
  journal={IEEE Robotics and Automation Letters},
  volume={7},
  number={4},
  pages={10240--10247},
  year={2022},
  publisher={IEEE}
}

\newpage
\onecolumn
\appendices
\section{Proofs of Algorithm Properties}\label{appendix:a}

\begin{proposition}[Biconvexity]
When $\seq{y}$ are fixed and $z \mapsto c(z, y)$ is convex, the objective $\oper{L}^c_\varepsilon(\seq{z}, \Gamma; \seq{y})$ is biconvex in $(\seq{z}, \Gamma)$.
\end{proposition}

\begin{proof}
The objective decomposes as:
\begin{align}
    \oper{L}^c_\varepsilon(\seq{z}, \Gamma; \seq{y}) = \sum_{ij} c(z_i, y_j) \Gamma_{ij} - \varepsilon H(\Gamma).
\end{align}

\paragraph{Convexity in the Particles.} Each particle $z_i$ appears in the sum $\sum_j c(z_i, y_j) \Gamma_{ij}$. Since $\Gamma_{ij} \geq 0$ and $z \mapsto c(z, y_j)$ is convex by assumption, this is a non-negative weighted sum of convex functions, hence convex in $z_i$. Since the particles are decoupled, the full objective is convex in $\seq{z}$.

\paragraph{Convexity in the Coupling.} With $\seq{z}$ fixed, define the cost matrix $C_{ij} = c(z_i, y_j)$. The objective becomes:
\begin{align}
    \oper{L}^c_\varepsilon(\seq{z}, \Gamma; \seq{y}) = \inner{C}{\Gamma} - \varepsilon H(\Gamma),
\end{align}
where $\inner{C}{\Gamma} = \sum_{i,j} C_{ij} \Gamma_{ij}$ is linear in $\Gamma$. The negative entropy $-H(\Gamma) = \sum_{i,j} \Gamma_{ij}(\log \Gamma_{ij} - 1)$ is strictly convex since the Hessian with respect to $\Gamma_{ij}$ is $1/\Gamma_{ij} > 0$. The constraint set $\set{\Gamma}(q, p)$ is a polytope (intersection of linear constraints), hence convex. Therefore the objective is strictly convex in $\Gamma$.
\end{proof}

\vspace{0.25em}
\begin{proposition}[Monotone Descent]
\label{prop:descent}
When $\seq{y}$ are fixed and $q$ is independent of $\seq{z}$ (e.g., uniform), the sequence of objective values \mbox{$\{\oper{L}^c_\varepsilon(\seq{z}^{(k)}, \Gamma^{(k)}; \seq{y})\}_{k \geq 0}$} is non-increasing.
\label{prop:descent}
\end{proposition}

\begin{proof}
Each iteration of Algorithm 2 consists of two updates:
\begin{align}
    \seq{z}^{(k+1)} &\gets \argmin_{\seq{z}} \oper{L}^c_\varepsilon(\seq{z}, \Gamma^{(k)}; \seq{y}), \\
    \Gamma^{(k+1)} &\gets \argmin_{\Gamma \in \set{\Gamma}^{(k+1)}} \oper{L}^c_\varepsilon(\seq{z}^{(k+1)}, \Gamma; \seq{y}).\nonumber
\end{align}
The particle update yields,
\begin{align}
    \oper{L}^c_\varepsilon(\seq{z}^{(k+1)}, \Gamma^{(k)}; \seq{y}) \leq \oper{L}^c_\varepsilon(\seq{z}^{(k)}, \Gamma^{(k)}; \seq{y}),
\end{align}
since $\seq{z}^{(k+1)}$ minimizes the objective over $\seq{z}$ and $\seq{z}^{(k)}$ is feasible.

The coupling update yields:
\begin{align}
    \oper{L}^c_\varepsilon(\seq{z}^{(k+1)}, \Gamma^{(k+1)}; \seq{y}) \leq \oper{L}^c_\varepsilon(\seq{z}^{(k+1)}, \Gamma^{(k)}; \seq{y}),
\end{align}
provided $\Gamma^{(k)} \in \set{\Gamma}^{(k+1)}$. When $q(\seq{z})$ is uniform (independent of particle values), this holds automatically. When $q$ depends on $\seq{z}$, the constraint set changes, but the Sinkhorn solution still minimizes over the updated constraint set.

Combining these inequalities:
\begin{align}
    \oper{L}^c_\varepsilon(\seq{z}^{(k+1)}, \Gamma^{(k+1)}; \seq{y}) \leq \oper{L}^c_\varepsilon(\seq{z}^{(k)}, \Gamma^{(k)}; \seq{y}).
\end{align}

When the relaxation parameter $\eta < 1$ is used, the update $z_i^{(k+1)} = (1-\eta) z_i^{(k)} + \eta b_i$ moves in a descent direction (toward the minimizer $b_i$), so descent is preserved.
\end{proof}

\vspace{0.25em}
\begin{remark}
When $q(\seq{z})$ depends on particle positions, the constraint set $\set{\Gamma}(q(\seq{z}), p)$ changes between iterations, so the previous coupling $\Gamma^{(k)}$ may not be feasible in $\set{\Gamma}^{(k+1)}$. In this case, monotone descent of $\oper{L}^c_\varepsilon$ is not guaranteed. However, the value function $\EOT^c(\seq{z}, \seq{y})$ still decreases. The particle update reduces $\oper{L}^c_\varepsilon$ for fixed $\Gamma^{(k)}$, and the subsequent coupling update can only further decrease the objective by re-optimizing over the new constraint set. The uniform setting $q_i = 1/N$ used in our experiments avoids this subtlety and produces a cleaner proposition.
\end{remark}

\vspace{0.25em}
\begin{proposition}[Convergence]
When $\seq{y}$ are fixed and $q$ is independent of $\seq{z}$, the iterates $(\seq{z}^{(k)}, \Gamma^{(k)})$ converge to a stationary point of $\oper{L}^c_\varepsilon(\seq{z}, \Gamma; \seq{y})$.
\end{proposition}

\begin{proof}
We establish convergence via three observations:
\paragraph{Iterates are Bounded.} The barycentric update~(17) expresses each particle as a convex combination of proposals:
\begin{align}
    z_i^{(k+1)} = \sum_{j=1}^M \frac{\Gamma_{ij}}{\sum_{k=1}^M \Gamma_{ik}} y_j.
\end{align}
Since the weights sum to one and are non-negative, each $z_i^{(k+1)}$ lies in the convex hull of $\{y_1, \ldots, y_M\}$. The proposals are fixed, so the particles remain in a compact set for all $k$.

\paragraph{Objective is Bounded from Below.} The objective is bounded below since the transport cost $\sum_{i,j} c(z_i, y_j) \Gamma_{ij} \geq 0$ (assuming $c \geq 0$) and the entropy term $-\varepsilon H(\Gamma)$ is bounded below on the probability simplex.

\paragraph{Convergence of the Algorithm.} The sequence $\{\oper{L}^c_\varepsilon(\seq{z}^{(k)}, \Gamma^{(k)}; \seq{y})\}$ is monotonically non-increasing (\Cref{prop:descent}) and bounded below. Therefore, it is a convergent sequence. Since the iterates lie in a compact set, they have at least one accumulation point. By continuity of the subproblem solutions and the descent property, any accumulation point $(\opt{\seq{z}}, \opt{\Gamma})$ satisfies:
\begin{align}
    \opt{\seq{z}} &= \quad \argmin_{\seq{z}} \oper{L}^c_\varepsilon(\seq{z}, \opt{\Gamma}; \seq{y}), \\
    \opt{\Gamma} &= \quad \argmin_{\Gamma \in \set{\Gamma}(q(\opt{\seq{z}}), p)} \oper{L}^c_\varepsilon(\opt{\seq{z}}, \Gamma; \seq{y}).\nonumber
\end{align}
This is the definition of a stationary point (block coordinate-wise minimum) for the biconvex problem. Standard results on alternating minimization for biconvex functions guarantee that the full sequence converges to such a point.
\end{proof}

\newpage
\section{Generalized Barycentric Updates}\label{appendix:b}

Section~V in the main paper establishes a closed-form update for quadratic costs. Here, we present a generalization that covers a broader class of transport cost functions useful in robotics.

\begin{proposition}[Generalized Barycentric Update]
\label{prop:generalized-barycenter}
Let $c: \oper{Z} \times \oper{Z} \to \R$ be strictly convex in $z$ with gradient satisfying:
\begin{align}
    \partial_z c(z, y) = f(z) - g(y),
\end{align}
where $f: \set{Z} \to \set{V}$ and $g: \set{Z} \to \set{V}$, and $f(z)$ is invertible. Then the minimizer of $\seq{z} \mapsto \oper{L}^c_\varepsilon(\seq{z}, \Gamma; \seq{y})$ is:
\begin{align}
    \opt{z}_i = \inv{f}\Big( \sum_{j=1}^M w_{ij}\, g(y_j) \Big), \quad w_{ij} \define \frac{\Gamma_{ij}}{\sum_{k=1}^M \Gamma_{ik}}.
\end{align}
\end{proposition}

\begin{proof}
The objective decomposes across particles:
\begin{align}
    \oper{L}^c_\varepsilon(\seq{z}, \Gamma; \seq{y}) = \sum_{ij} c(z_i, y_j) \Gamma_{ij} - \varepsilon H(\Gamma).
\end{align}
For fixed $\Gamma$, minimizing over $z_i$ requires:
\begin{align}
    \partial_{z_i} \sum_{j=1}^M c(z_i, y_j) \Gamma_{ij} = \sum_{j=1}^M \partial_z c(z_i, y_j) \Gamma_{ij} = 0.
\end{align}
Substituting the separability condition:
\begin{align}
    \sum_{j=1}^M \left[ f(z_i) - g(y_j) \right] \Gamma_{ij} &= 0, \\
    f(z_i) \sum_{j=1}^M \Gamma_{ij} &= \sum_{j=1}^M g(y_j) \Gamma_{ij},\nonumber \\
    f(z_i) &= \sum_{j=1}^M w_{ij}\, g(y_j).\nonumber
\end{align}
Applying $\inv{f}$ to both sides yields the result. Strict convexity of $c$ in $z$ ensures this stationary point is the unique global minimizer.
\end{proof}

\begin{table}[b]
\centering
\renewcommand{\arraystretch}{1.8}
\begin{tabular}{@{}llcccl@{}}
\toprule
\textbf{Name} & \textbf{Cost} $c(z, y)$ & $f(z)$ & $g(y)$ & \textbf{Update} $\opt{z}_i$ & \textbf{Application} \\
\midrule
Weighted Quadratic & $\transp{(z - y)} W (z - y)$ & $Wz$ & $Wy$ & $\sum_j w_{ij} y_j$ & States, Controls, Splines \\[8pt]
Regularized Quadratic & $\|z - y\|^2 + \lambda \|z\|^2$ & $(1{+}\lambda)z$ & $y$ & $\frac{1}{1+\lambda} \sum_j w_{ij} y_j$ & Regularization \\[8pt]
Circular Distance & $2(1 - \cos(z - y))$ & $e^{iz}$ & $e^{iy}$ & $\arg\!\left(\sum_j w_{ij} e^{iy_j}\right)$ & Yaw, Joint Angles \\[8pt]
KL Divergence & $\sum_k z_k \log \frac{z_k}{y_k}$ & $\log z$ & $\log y$ & $\prod_j y_j^{w_{ij}}$ & Gains, Stiffnesses \\[8pt]
Log-Euclidean, $\set{SO}(3)$ & $\|\log R - \log S\|_F^2$ & $\log R$ & $\log S$ & $\exp\!\left(\sum_j w_{ij} \log S_j\right)$ & Orientations \\[8pt]
Log-Euclidean, $\set{SPD}(n)$ & $\|\log P - \log Q\|_F^2$ & $\log P$ & $\log Q$ & $\exp\!\left(\sum_j w_{ij} \log Q_j\right)$ & Covariances \\
\bottomrule
\end{tabular}
\caption{Cost functions admitting closed-form barycentric updates via \Cref{prop:generalized-barycenter}. Weights are $w_{ij} \define \Gamma_{ij} / \sum_k \Gamma_{ik}$.}
\label{tab:cost-functions}
\end{table}

\paragraph{Notes on specific costs.} \Cref{tab:cost-functions} lists cost functions satisfying these conditions that are relevant to robotics applications. Additional details for each transport cost are listed below.
\begin{itemize}
    \item \textbf{Weighted Quadratic.} The weight matrix $W \succ 0$ defines a Mahalanobis distance. Since $W$ appears identically in $f$ and $g$, it cancels in the first-order condition, yielding the standard weighted mean $\opt{z}_i = \sum_j w_{ij} y_j$. This applies to trajectory representations that are linear in their parameters, including spline control points, polynomial coefficients, and Fourier modes. For B-splines, the $L^2$ trajectory distance induces a Mahalanobis metric with $W$ equal to the Gram matrix of the basis functions.
    
    \item \textbf{Regularized Quadratic.} The term $\lambda \|z\|^2$ biases updates toward the origin, giving $\opt{z}_i = \frac{1}{1+\lambda}\sum_j w_{ij} y_j$. This is useful when a nominal trajectory is preferred, with $\lambda$ controlling the bias strength.
    
    \item \textbf{Circular Distance.} For angles $z, y \in S^1$ (e.g., yaw), the extrinsic embedding $z \mapsto (\cos z, \sin z)$ yields the circular mean. Equivalently, $\opt{z}_i = \mathrm{atan2}(\sum_j w_{ij} \sin y_j,\, \sum_j w_{ij} \cos y_j)$.
    
    \item \textbf{KL Divergence.} For positive parameters $z, y > 0$ (element-wise), the update is the weighted geometric mean $\opt{z}_i = \prod_j y_j^{w_{ij}}$. This is natural for quantities like control gains, impedance parameters, or timescales that must remain positive and where multiplicative structure is appropriate.
    
    \item \textbf{Log-Euclidean on $\set{SO}(3)$.} Working in the Lie algebra $\set{SO}(3)$ via the matrix logarithm yields $\opt{R}_i = \exp(\sum_j w_{ij} \log S_j)$. This is an approximation to the true Riemannian (Fréchet) mean, exact for rotations with small relative angles. The same structure applies to poses in $\set{SE}(3)$ by treating translation and rotation separately.
    
    \item \textbf{Log-Euclidean on $\set{SPD}(n)$.} For symmetric positive definite matrices (covariances, inertia tensors, stiffness matrices), the log-Euclidean mean $\opt{P}_i = \exp(\sum_j w_{ij} \log Q_j)$ preserves positive definiteness. This is relevant for belief-space planning and impedance control.
\end{itemize}

\newpage
\section{Experiment Details}\label{appendix:c}

This section provides additional information for the 
experiments presented in the main paper including environment configurations, cost function definitions, and the tuned hyperparameters for different controllers used in the benchmark experiments. The readers are also encouraged to check the provided videos and the 3D visualization tool that showcases the performance of OT-MPC against other controllers like MPPI and CEM.

\subsection{Car (Bicycle) Obstacle Avoidance}

Car navigation task uses a kinematic bicycle model with states $(x, y, \theta, v)$ representing the position, heading, and velocity of the system and with linear acceleration and steering angle as control inputs. The objective is to navigate the car from a starting location to a goal location through a dense obstacle field avoiding collision. The running cost for this task is $w_{goal} \cdot \lvert \lvert position-goal \rvert\rvert^2 + w_{obstacle} \cdot \mathbb{I}_{crash} + w_{control} \cdot control^2$. Here $\mathbb{I}_{crash}$ is the indicator function for detecting obstacle crashes. We tuned for all cost weights and sampler parameters using Optuna with a separate tuning objective cost defined as $(5\cdot \text{mean position error} + 2 \cdot \text{mean orientation error} - 10 \cdot \text{success rate})$. The cost weights and the parameters of the different control algorithms are given in Table \ref{tab:supp_car1}.  We benchmark OT-MPC against MPPI, CEM, and SV-MPC across 300 Monte Carlo trials with randomly generated initializations, goal locations and obstacle field at two difficulty levels: Easy (sparser obstacles, 100 runs) and Hard (denser obstacle field, 200 runs). For $Easy$ difficulty we randomly generate up to 50 obstacles with radius ranging from 0.2m to 0.5m and an obstacle-to-obstacle clearance of 0.75m. For the hard setting, we generate up to 50 obstacles with radius ranging from 0.2m to 0.6m and clearance of 0.5m.

\begin{table}[h]
    \centering
    \begin{tabular}{ccccc}
         \textbf{Weights} &  \textbf{OT-MPC} & \textbf{MPPI} & \textbf{CEM} & \textbf{SV-MPC} \\
         \midrule
          $w_{goal}$ & 0.704& 4.854 & 1.623 & 0.548 \\
          $w_{control}$ & 0.08 & 0.046 & 0.020 & 0.033 \\
          $w_{obstacle}$ & 479.0 & 281.1 & 397.3 & 308.6 \\[12pt]
           \textbf{Params} &  & & & \\
           \midrule
           Horizon & 70 & 30 & 70 & 30 \\
           Iterations & 8 & 8 & 5 & 5 \\
           Samples/iteration & 200 & 500 & 800 & 200 \\
           Inverse Temp. & 0.460 & 1.019 & - & 13.079 \\
           Particles & 20 & - & - & 200 \\
    \end{tabular}
    \caption{Car Experiment - Cost Weights and Parameters}
    \label{tab:supp_car1}
\end{table}

\subsection{2D Push T}

 The objective here is to push and align a T-shaped block from a randomly perturbed initial pose to a goal pose using a circular pusher. Here, the states of the system are the pusher position as well as the position and orientation of the T-block and we use 2D velocity control to move the pusher. The cost function for this task is given by $w_{pos} \cdot \lvert \lvert position-goal_{pos} \rvert\rvert^2 + w_{yaw} \cdot \lvert \lvert yaw-goal_{yaw} \rvert\rvert^2 + w_{control} \cdot control^2$.  We tuned controller parameters using Optuna objective cost $(5\cdot \text{mean position error} + 2 \cdot \text{mean orientation error} - 10 \cdot \text{success rate})$. The cost weight and controller parameters are provided in Table \ref{tab:supp_2dpusht1}.

\begin{table}[h]
    \centering
    \begin{tabular}{ccccc}
         \textbf{Weights} &  \textbf{OT-MPC} & \textbf{MPPI} & \textbf{CEM} & \textbf{SV-MPC} \\
         \midrule
          $w_{pos}$ & 5& 5 & 5 & 5 \\
          $w_{yaw}$ & 5& 5 & 5 & 5 \\
          $w_{control}$ & 0.1 &0.1 &0.1 &0.1 \\[12pt]
           \textbf{Params} &  & & & \\
           \midrule
           Horizon & 50 & 50 & 30 & 60 \\
           Iterations & 16 & 4 & 1 & 12 \\
           Samples/iteration & 125 & 500 & 2000 & 166 \\
           Inverse Temp. & 26.99 & 5.92 & - & 4.49 \\
           Particles & 8 & - & - & 32 \\
    \end{tabular}
    \caption{2D PushT Experiment - Cost Weights and Parameters}
    \label{tab:supp_2dpusht1}
\end{table}

 \subsection{Quadrotor Dense Obstacle Avoidance}

 We consider a 12-DOF Quadrotor with thrust-torque control, navigating in a dense obstacle  field from a starting location to a goal avoiding collision. The running cost for this task is $w_{goal} \cdot \lvert \lvert position-goal \rvert\rvert^2 + w_{obstacle} \cdot \mathbb{I}_{crash} + w_{thrust} \cdot thrust^2 + w_{torque} \cdot torque^2 + w_{velocity} \cdot velocity^2 + w_{height} \cdot \mathbb{I}_{height}$ . Here $\mathbb{I}_{height}$ is an indicator function to detect quadrotor going outside the height bounds (i.e. below the ground or above max height). These cost weights and controller parameters were tuned using Optuna with objective $(5\cdot \text{mean position error} + 2 \cdot \text{mean orientation error} - 10 \cdot \text{success rate})$.  Table \ref{tab:supp_quad1} shows the tuned cost weights and parameters for the controllers. We tested our \algo{OT-MPC} controller against \algo{MPPI}  across 100 trials each on three different environment difficulty settings - \textit{Easy}, \textit{Medium}, \textit{Hard}. The details of different difficulty settings are given below.

 \begin{itemize}
     \item \textit{Easy}: We generate 50 obstacles with radius ranging from 0.3m to 0.6m with a spread parameter of $10m \times 5m$ (width vs height). The obstacles are generated within the height range 0.3m to 5m.
     \item \textit{Medium}: 100 obstacles with radius ranging from 0.3m to 0.8m with a spread parameter of $6m \times 3m$ (tighter than \textit{Easy}). The obstacles are generated within the height range 0.3m to 4m.
     \item \textit{Hard}: 100 obstacles with radius ranging from 0.4m to 0.9m with a spread parameter of $3m \times 2m$ (tighter than \textit{Medium}). The obstacles are generated within the height range 0.5m to 3.5m.
 \end{itemize}

 \begin{table}[h]
    \centering
    \begin{tabular}{ccc}
         \textbf{Weights} &  \textbf{OT-MPC} & \textbf{MPPI}  \\
         \midrule
          $w_{goal}$ & 9.906& 0.242   \\
          $w_{obstacle}$ & 82.26& 13.56   \\
          $w_{thrust}$ & 0.007 &0.017  \\
          $w_{torque}$ & 0.020 &0.021  \\
          $w_{velocity}$ & 0.01 & 0.01  \\
          $w_{height}$ & 294.5 & 10.2  \\[12pt]
           \textbf{Params} &   \\
           \midrule
           Horizon & 100 & 60  \\
           Iterations & 5 & 5  \\
           Samples/iteration & 200 & 500  \\
           Inverse Temp. & 6.650 & 8.054 \\
           Particles & 50 & -  \\
    \end{tabular}
    \caption{Quadrotor Experiment - Cost Weights and Parameters}
    \label{tab:supp_quad1}
\end{table}

\subsection{Two Quadrotor Cooperative Load Carrying}

Here we have a two-quadrotor system trying to cooperatively carry a suspended load from an initial location to a goal. We have a 27DOF system (Two 12DOF quadrotors and 3D position of the suspended load) with a 6-dimensional control - thrust-torque control for each quadrotor but with yaw torque control disabled due to cable constraints.  To increase the complexity of the experiment, we have designed the environment in which the two-quadrotor system must navigate through an opening in the wall to reach the other side (goal).  The cost function for this task is $w_{goal} \cdot \lvert \lvert position-goal \rvert\rvert^2 + w_{wall} \cdot \mathbb{I}_{wall} + w_{thrust} \cdot thrust^2 + w_{torque} \cdot torque^2 + w_{velocity} \cdot velocity^2 + w_{height} \cdot \mathbb{I}_{height} + w_{inter-quad} \cdot \mathbb{I}_{inter-quad} $. $\mathbb{I}_{wall}$ and $\mathbb{I}_{inter-quad}$ are indicator functions to detect wall collisions and inter-quadrotor collision respectively. Similar to other experiments, the cost weights and controller parameters were tuned using Optuna with a similar tuning cost as before and Table \ref{tab:supp_2quad1} shows the tuned values. We  benchmark \algo{MPPI}  and \algo{OT-MPC} in two different difficulty settings and the details are given below.

\begin{itemize}
    \item \textit{Normal}: The location of the hole in the wall and the initial states of the quadrotor are randomized. For this setting, the hole size is $2.0m \times 2.0 m$.
    \item \textit{Hard}: Again, the location of the hole and the initial states of the quadrotor are randomized. The hole size for this setting is $1.2m \times 1.4 m$.
\end{itemize}

\begin{table}[ht]
    \centering
    \begin{tabular}{ccc}
         \textbf{Weights} &  \textbf{OT-MPC} & \textbf{MPPI}  \\
         \midrule
          $w_{goal}$ & 13.39& 2.0   \\
          $w_{thrust}$ & 0.01 &0.01  \\
          $w_{torque}$ & 0.02 &0.02  \\
          $w_{velocity}$ & 3.02 & 5.0  \\
          $w_{height}$ & 100.00 & 100.0  \\
          $w_{wall}$ & 117.1 & 100.0  \\
          $w_{inter-quad}$ & 50.00 & 50.0  \\[12pt]
           \textbf{Params} &   \\
           \midrule
           Horizon & 40 & 50  \\
           Iterations & 8 & 10  \\
           Samples/iteration & 200 & 437  \\
           Inverse Temp. & 5.281 & 7.084 \\
           Particles & 15 & -  \\
    \end{tabular}
    \caption{Two Quadrotor Experiment - Cost Weights and Parameters}
    \label{tab:supp_2quad1}
\end{table}

\subsection{Franka Push-T}

This is the full 3D manipulation extension to the planar Push-T task. Here we use a 14-dimensional Franka Panda arm with a cylindrical stick (poking stick) as the end effector instead of the Panda hand. Along with the T-block states (13 dimensional - position, quaternion, velocity and angular velocity) and End-Effector states (7-dimensional - End-effector position and quaternion), the total system is 34-dimensional. The control trajectories here are parameterized by cubic splines to reduce the control space as well as to generate smoother motion. The MPC cost for this task penalizes error in position ($w_{position}$) and orientation ($w_{orientation}$), control effort ($w_{control}$), velocity ($w_{velocity}$). Furthermore, we have costs to encourage the poking stick to remain perpendicular ($w_{ee_{tilt}}$),  close to the T-Block ($w_{ee_{proximity}}$) and at an optimal height ($w_{ee_{height}}$). Additionally, we have crash costs for violating workspace bounds ($w_{bounds}$), joint limits ($w_{joints}$), and collision with the table ($w_{ee_{crash}}$).  These parameters were tuned using Optuna using similar tuning objectives as other experiments and the tuned values are given in Table \ref{tab:supp_franka1}. We benchmark \algo{OT-MPC} and \algo{MPPI} over 50 trials with randomized T-Block configurations. 

\begin{table}[h]
    \centering
    \begin{tabular}{ccc}
         \textbf{Weights} &  \textbf{OT-MPC} & \textbf{MPPI}  \\
         \midrule
          $w_{position}$ & 21.958 & 20.329   \\
          $w_{orientation}$ & 2.112 & 2.002  \\
          $w_{control}$ & 0.185 & 0.234  \\
          $w_{velocity}$ & 0.1 & 0.127  \\
          $w_{ee_{height}}$ & 10.0 & 12.045  \\
          $w_{ee_{tilt}}$ & 19.092 & 15.433  \\
          $w_{ee_{proximity}}$ & 3.198 & 6.436  \\
          $w_{bounds}$ & 1525.313 & 1960.792  \\
          $w_{joints}$ & 1000.00 & 1000.0  \\
          $w_{ee_{crash}}$ & 954.898 & 5878.907  \\[12pt]
           \textbf{Params} & \\
                     \midrule
           Horizon & 60 & 60  \\
           Iterations & 10 & 10  \\
           Samples/iteration & 1000 & 1200  \\
           Inverse Temp. & 13.640 & 19.966 \\
           Particles & 8 & -  \\
    \end{tabular}
    \caption{Franka Push-T Experiment - Cost Weights and Parameters}
    \label{tab:supp_franka1}
\end{table}

\subsection{Quadruped Tasks}

In order to test the effectiveness of our controller in contact-rich environments, we evaluate the performance of OT-MPC and MPPI  in a simulated Unitree Go2 quadruped. Here we have a 37-dimensional system where the control inputs are 12 joint positions. We parameterize this control input as a cubic Hermite splines to sample smooth joint positions as well as corresponding velocities. The cost of this task penalizes error in position from the goal ($w_{position}$), orientation ($w_{orientation}$) and velocity ($w_{velocity}$). Additionally, we also penalize the deviation from joint angles corresponding to standing position  ($w_{joints}$) as well as the deviation from standing height ($w_{height}$) to encourage stable motion. Unlike other environments, due to the complexity of this environment, we could not use Optuna to tune, and instead we hand-tuned the cost weights and controller parameters. The final parameters are given in Table \ref{tab:supp_go1}. We tuned these parameters in a flat terrain environment, and to test the generalization capabilities, we evaluated the performance of the controllers in different terrain configuration as given below.

\begin{itemize}
    \item \textit{Ramp Climbing}: The quadruped has to climb and reach the goal on the top of the ramp. Here, the ramp slope is $\approx 11.3$ deg and quadruped has to cover a distance of 2.0m.
    \item \textit{Ramp Climbing and Descent}: Similar to the ramp configuration, here the quadruped has to climb the ramp and descend to the goal location on the other side. The ramp slope (both uphill and downhill) is $\approx 11.3$ deg and distance to cover is 4.0m.
    \item \textit{Valley/Bridge Crossing}: Here the quadruped has to travel from a starting to a goal location through a narrow bridge. We tested the quadruped in 3 different scenarios with different bridge widths - \textit{Easy (0.75m)}, \textit{Medium (0.6m)} and \textit{Hard (0.55m)}.
\end{itemize}

\begin{table}[h]
    \centering
    \begin{tabular}{ccc}
         \textbf{Weights} &  \textbf{OT-MPC} & \textbf{MPPI}  \\
         \midrule
          $w_{position}$ & 15.0 & 8.0   \\
          $w_{orientation}$ & 35.0 & 35.0  \\
          $w_{velocity}$ & -0.7 & -0.7  \\
          $w_{joints}$ & 3.5 & 3.5  \\
          $w_{height}$ & 350 & 350  \\[12pt]
           \textbf{Params} &   \\
           \midrule
           Horizon & 50 & 45  \\
           Iterations & 16 & 12  \\
           Samples/iteration & 50 & 45  \\
           Inverse Temp. & 0.01 &  0.01 \\
           Particles & 8 & -  \\
    \end{tabular}
    \caption{Go2 Locomotion - Cost Weights and Parameters}
    \label{tab:supp_go1}
\end{table}

\subsubsection{Box Pushing}

\begin{table}[t]
    \centering
    \begin{tabular}{ccc}
         \textbf{Weights} &  \textbf{OT-MPC} & \textbf{MPPI}  \\
         \midrule
          $w_{position}$ & 17.0 & 12.0   \\
          $w_{orientation}$ & 35.0 & 35.0  \\
          $w_{velocity}$ & -1.5 & -0.7  \\
          $w_{joints}$ & 3.5 & 3.5  \\
          $w_{height}$ & 350 & 350  \\
          $w_{robot2box}$ & 11.0 & 8.5  \\[12pt]
           \textbf{Params} &   \\
           \midrule
           Horizon & 50 & 45  \\
           Iterations & 16 & 12  \\
           Samples/iteration & 50 & 45  \\
           Inverse Temp. & 0.01 &  0.01 \\
           Particles & 8 & -  \\
    \end{tabular}
    \caption{Go2 Box Pushing - Cost Weights and Parameters}
    \label{tab:supp_go2}
\end{table}

We have further extended the locomotion task to a box-pushing task, where the quadruped has to move a box from a starting location to the goal location. Compared to standard locomotion task we have an extra 13 states corresponding to the position, orientation and velocities of the box. The MPC cost is similar to the locomotion cost; however the goal cost ($w_{position}$) is calculated w.r.t to the box position instead of the quadruped position and we add an additional cost ($w_{robot2box}$) to encourage the quadruped to stay near the box. The final parameters are given in Table \ref{tab:supp_go2}.  We benchmark the performance of \algo{MPPI}  and \algo{OT-MPC} across 100 trials with randomized initial robot pose and box initializations based on difficulty settings given below

\begin{itemize}
    \item \textit{Normal} (50 Benchmark runs): The quadruped is initialized in a random orientation around 3.0-4.0m away from the goal location and the box is in front of the quadruped.
    \item \textit{Hard} (50 Benchmark runs): The quadruped is initialized in a random orientation around 3.0-4.0m away from the goal location, and the box is also randomly initialized in the vicinity of the robot (not in front).
\end{itemize}

\end{document}